\documentclass{article}

\usepackage[utf8]{inputenc} 
\usepackage[T1]{fontenc}    
\usepackage{hyperref}       
\usepackage{url}            
\usepackage{booktabs}       
\usepackage{amsfonts}       
\usepackage{nicefrac}       
\usepackage{microtype}      

\usepackage{epsfig}
\usepackage{graphicx}
\usepackage{amsmath}
\usepackage{amssymb}
\usepackage{enumerate}
\usepackage{xcolor}
\usepackage{multirow}

\newcommand{\tr}{\mbox{$^{\top}$}}

\def\R{{\rm I} \! {\rm R}}

\newcommand{\eq}[1]{(\ref{eq:#1})}

\newcommand{\fig}[1]{fig~\ref{fig:#1}}

\newcommand{\sect}[1]{section~\ref{sec:#1}}

\newcommand{\pro}[1]{Proposition~\ref{pro:#1}}

\newcommand{\SKIP}[1]{} 

\newcommand{\mbegin} {\left [ \begin{array}}

\newcommand{\mend}   {\end{array} \right ]}

\newcommand{\detbegin} {\left | \begin{array}}
\newcommand{\detend}   {\end{array} \right |}

\newcommand{\vbegin} {\left ( \begin{array}{c}}

\newcommand{\vend} {\end{array}\right )}

\def\squareforqed{\hbox{\rlap{$\sqcap$}$\sqcup$}}

\def\qed{\ifmmode\squareforqed\else{\unskip\nobreak\hfil
	\penalty50\hskip1em\null\nobreak\hfil\squareforqed
	\parfillskip=0pt\finalhyphendemerits=0\endgraf}\fi}

\def\vec#1{\mathchoice%
	{\mbox{\bf $\displaystyle\bf#1$}}
	{\mbox{\bf $\textstyle\bf#1$}}
	{\mbox{\bf $\scriptstyle\bf#1$}}
	{\mbox{\bf $\scriptscriptstyle\bf#1$}}}
\def\v#1{\protect\vec #1}

\newcommand{\showeqnlabel}{
	\hbox to 0pt{\quad\quad\relax\fbox{\scriptsize\rm\eqnlblx}%
	\gdef\eqnlblx{xxxx}}} \newcommand{\eqnlblx}{}

\def\@eqnnum{\rm (\theequation)\showeqnlabel}

\newcommand{\nofig}[1]{\centerline{\bf Figure here}}

\def\mat#1{\mathchoice{\mbox{\bf$\displaystyle\tt#1$}}
	{\mbox{\bf$\textstyle\tt#1$}}
	{\mbox{\bf$\scriptstyle\tt#1$}}
	{\mbox{\bf$\scriptscriptstyle\tt#1$}}}
\def\m#1{\protect\mat #1}

\def\etal{\emph{et al}.}

\def\tabref#1{Table~\ref{#1}}

\newif\ifsupp
\suppfalse

\newif\ifarxiv
\arxivfalse

\newif\iffinal
\finalfalse
\def\finalcopy{\global\finaltrue}

\newcommand{\myref}[1]{{\color{red}#1}}

\usepackage[page]{appendix}

\newcommand{\calD}{{\cal D}}
\newcommand{\out}{[0,1]}

\newcommand{\hatx}{\hat{x}}
\newcommand{\haty}{\hat{y}}

\newcommand{\hatvy}{\hat{\v y}}

\newcommand{\calK}{{\cal K}}
\newcommand{\figdir}{figures}

\newcommand{\citenew}[1]{(\cite{#1})}

\input{theorenv}

\definecolor{orange}{rgb}{1,0.5,0}

\newcommand{\aj}[1]{{\color{magenta}\textbf{AJ: }#1}}

\newcommand{\rih}[1]{{\color{green}\textbf{RIH: }#1}}

\finalcopy 

\ifarxiv
\ifsupp
\global\suppfalse
\fi
\fi

\PassOptionsToPackage{numbers, compress}{natbib}
\iffinal 	
     \usepackage{iclr2021_conference,times}
     \iclrfinalcopy
\else 	
\ifarxiv	
     \usepackage[preprint]{iclr2021_conference,times}
\else 	
    \usepackage{iclr2021_conference,times}
\fi     
\fi

\ifsupp
\title{Calibration of Neural Networks using Splines Supplementary Material}
\else
\title{Calibration of Neural Networks using \\ Splines}
\fi

\author{
  Kartik Gupta\textsuperscript{1,2}, Amir Rahimi\textsuperscript{1},
  Thalaiyasingam Ajanthan\textsuperscript{1},
  Thomas Mensink\textsuperscript{3},
  \\ \hspace*{2pt}\textbf{Cristian Sminchisescu\textsuperscript{3}, Richard Hartley\textsuperscript{1,3}}\\
  \hspace*{2pt}\textsuperscript{1}Australian National University, \textsuperscript{2}Data61, CSIRO, \textsuperscript{3}Google Research\\ \texttt{\{kartik.gupta,amir.rahimi,thalaiyasingam.ajanthan\}@anu.edu.au}\\
  \texttt{\{mensink,sminchisescu,richardhartley\}@google.com}
}

\begin{document}

\maketitle

\ifsupp
\appendix
Here, we first provide the proof of our main result, discuss more about top-$r$ calibration and spline-fitting, and then turn to additional experiments.
\section{Proof of Proposition 4.1}
We first restate our proposition below.
\begin{proposition}\label{pro:spline-deriv1}
If $h(t) = P(Y = k, f_k(X) \le s(t))$ as in (\myref{14}) of the main paper where $s(t)$ 
is the $t$-th fractile score. Then 
$h'(t) = P(Y = k ~|~ f_k(X) = s(t))$, where $h'(t) = dh/dt$.
\end{proposition}
\begin{proof}
The proof is using the fundamental relationship between the Probability Distribution Function (PDF) and the Cumulative Distribution Function (CDF) and it is provided here for completeness.
Taking derivatives, we see (writing $P(k)$ instead of $P(Y=k)$):
\begin{align}
\begin{split}
h'(t) &= P(k, f_k(X) = s(t)) ~.~ s'(t) \\
       &= P(k ~|~ f_k(X) = s(t)) ~.~ P(f_k(X) = s(t)) ~.~ s'(t) \\
       &= P(k ~|~ f_k(X) = s(t)) ~.~\frac{d}{dt} \big(P(f_k(X) \le s(t))\big) \\
       &= P(k ~|~ f_k(X) = s(t))~.~ \frac{d}{dt} (t) \\
       &= P(k ~|~ f_k(X) = s(t)) ~.
\end{split}
\end{align}
The proof relies on the equality $P(f_k(X) \le s(t)) = t$.
In words:
$s(t)$ is the 
value that a fraction $t$ of the scores are less than or equal. This equality
then says: the probability that a score is less than or equal to the value that a fraction
$t$
of the scores lie below, is (obviously) equal to $t$.
\end{proof}


\section{More on top-$r$ and within-top-$r$ Calibration}
In the main paper, definitions of top-$r$ and within-top-$r$ 
calibration are given in equations (\myref{4}) and (\myref{5}).  Here, a few more
details are given of how to calibrate the classifier $f$ for 
top-$r$ and within-top-$r$ calibration.

The method of calibration using splines described in this paper
consists of fitting a spline to the cumulative accuracy,
defined as $h_i$ in equation (\myref{11}) in the main paper.
For top-$r$ classification, the method is much the same
as for the classification for class $k$. Equation (\myref{11})
is replaced by sorting the data according to the 
$r$-th top score, then defining
\begin{align}
\begin{split}
\label{eq:cumsum-topr}
\tilde{h}_0 = h_0 &= 0~, \\
h_i &= h_{i-1} + \v 1(y^{(-r)} = 1) /N~, \\
\tilde{h}_i &= \tilde{h}_{i-1} + f^{(-r)}(\v x_i) /N~,
\end{split}
\end{align}
where $y^{(-r)}$ and $f^{(-r)}(\v x_i)$ are defined in the main
paper, equation (\myref{3}).  These sequences may then be used
both as a metric for the correct top-$r$ calibration and 
for calibration using spline-fitting as described.

For within-top-$r$ calibration, one 
sorts the data according to the sum of the top $r$ scores,
namely $\sum_{s=1}^r f^{(-s)}(\v x_i)$, then computes
\begin{align}
\begin{split}
\label{eq:cumsum-within-topr}
\tilde{h}_0 = h_0 &= 0~, \\
h_i &= h_{i-1} + \v 1\Big(\sum_{s=1}^r y^{(-s)} = 1\Big) \Big/N~, \\
\tilde{h}_i &= \tilde{h}_{i-1} + \sum_{s=1}^r f^{(-s)}(\v x_i) /N~,
\end{split}
\end{align}
As before, this can be used as a metric, or as the starting point
for within-top-$r$ calibration by our method.  Examples of
this type of calibration (graphs for uncalibrated networks in \fig{main-graphswn2} and \fig{main-graphswn3}) is given in the graphs provided in  \fig{densenet-calibratedwn2} and \fig{densenet-calibratedwn3} for within-top-2 predictions and within-top-3 predictions respectively.

It is notable that if a classifier is calibrated in
the sense of equation (\myref{1}) in the main paper (also called 
multi-class-calibrated), then it is also
calibrated for top-$r$ and within-top-$r$ classification.


\section{Least Square Spline Fitting}
Least-square fitting using cubic splines is a known technique.  However,
details are given here for the convenience of the reader.
Our primary reference is \citenew{Mckinley_cubicspline}, which we
adapt to least-squares fitting.  We consider the case
where the knot-points are evenly spaced.

We change notation from that used in the main paper by denoting
points by $(x, y)$ instead of $(u, v)$.
Thus, given knot points $(\hatx_i, \haty_i)_{k=1}^K$ one is required to
fit some points $(x_i, y_i)_{i=1}^N$.  
Given a point $x$, the corresponding spline value is given
by $y = \v a(x)\tr \m M \hatvy$, where $\hatvy$ is the vector of
values $\haty_i$.  The form of the vector $\v a(x)$ and 
the matrix $\m M$ are given in the following.

The form of the matrix $\m M$ is derived from equation (25)
in \cite{Mckinley_cubicspline}. Define the matrices

\[
\m A = \mbegin{ccccccccc}
 4 & 1\\
1 & 4 & 1 & \\
 & 1 & 4 & 1 & \\
 &   &   & \ddots \\
 &   &   &       & 1 & 4 & 1 \\
 &   &   &       &   & 1 & 4
\mend
~; ~~~
\m B = \frac{6}{h^2}\mbegin{rrrrrrr}
1 & -2 & 1 & \\
  & 1 & -2 & 1 & \\
  &   &   & \ddots \\
  &   &   &       & 1 & -2 & 1
\mend~,
\]
where $h$ is the distance between the knot points.
These matrices are of dimensions $K-2\times K-2$ and $K-2 \times K$
respectively.  Finally, let $\m M$ be the matrix
\[
\m M = \mbegin{c}
\v 0_K\tr \\ \m A^{-1} \m B \\ \v 0_K\tr \\ \m I_{K\times K}
\mend~.
\]
Here, $\v 0_K$ is a vector of zeros of length $K$,
and $\m I_{K\times K}$ is the identity matrix.
The matrix $\m M$ has dimension $2K \times K$.

Next, let the point $x$ lie between
the knots $j$ and $j+1$ and let
$u = x - \hatx_j$.
Then define the vector 
$\v v = \v a(x)$ by
values
\begin{align*}
v_j &= -u^3 / (6h) + u^2 / 2 - h u / 3~,\\
v_{j+1} &= u^3 / (6h) - h u / 6 ~,\\
v_{j+K} &= -u / h + 1~,\\
v_{j+1+K} &= u/h~,
\end{align*}
with other entries equal to $0$.

Then the value of the spline is given by
\[
y = \v a(x)\tr \m M \hatvy~,
\]
as required.  This allows us to fit the spline
(varying the values of $\hatvy$) to points $(x_i, y_i)$
by least-squares fit, as described in the main paper.

The above description is for so-called {\em natural} (linear-runout)
splines.  For quadratic-runout or cubic-runout splines
the only difference is that the first and last rows
of matrix $\m A$ are changed -- see \cite{Mckinley_cubicspline}
for details.

As described in the main paper, it is also possible to
add linear constraints to this least-squares problem,
such as constraints on derivatives of the spline.
This results in a linearly-constrained quadratic programming
problem.


\section{Additional Experiments}
We first provide the experimental setup for different datasets in \tabref{tab:setup}. Note, the calibration set is used for spline fitting in our method and then final evaluation is based on an  unseen test set.

\begin{table}[t]
    \centering
    \small
    \begin{tabular}{lcccc}
        \toprule
        Dataset & Image Size & \# class & Calibration set & Test set \\
        \midrule
        CIFAR-10 & $32\times 32$ & $10$ & $5000$  & $10000$ \\
        CIFAR-100 & $32\times 32$ & $100$ & $5000$  & $10000$ \\
        SVHN & $32\times 32$ & $10$ & $6000$ & $26032$ \\
        ImageNet & $224\times 224$ & $1000$ & $25000$ & $25000$\\
        \bottomrule
    \end{tabular}
    \vspace{1ex}
    \caption{\small \em Dataset splits used for all the calibration experiments. Note, ``calibration'' set is used for spline fitting in our method and calibration for the baseline methods and then different methods are evaluated on ``test'' set.}
    \label{tab:setup}
    \vspace{-2ex}
    \end{table}
\SKIP{
\begin{table}[t!]
\scriptsize
\begin{tabular}{ll|cccccc}
\toprule
Dataset &
  Model &
  Uncalibrated &
  Temp. Scaling &
  Vector Scaling &
  MS-ODIR &
  Dir-ODIR &
  \textbf{Ours (Spline)} \\ \midrule
\multirow{5}{*}{CIFAR-10}  & Resnet-110     & 0.01805                          & \textbf{0.00097}                  & 0.00176                            & {\underline{ 0.00140}}               & 0.00195                      & 0.00742                                                     \\
                           & Resnet-110-SD  & 0.01423                          & 0.00111                           & 0.00089                            & {\underline{ 0.00082}}               & \textbf{0.00073}             & 0.01087                                                     \\
                           & DenseNet-40    & 0.02256                          & 0.00435                           & 0.00409                            & {\underline{ 0.00395}}               & \textbf{0.00348}             & 0.01718                                                     \\
                           & Wide Resnet-32 & 0.01812                          & 0.00145                           & \textbf{0.00105}                   & {\underline{ 0.00124}}               & 0.00139                      & 0.00808                                                     \\
                           & Lenet-5        & 0.03545                          & 0.00832                           & 0.00831                            & \textbf{0.00631}            & 0.00804                      & {\underline{ 0.00686}}                                               \\ \midrule
\multirow{5}{*}{CIFAR-100} & Resnet-110     & 0.14270                          & 0.00885                           & {\underline{ 0.00649}}                      & 0.01425                     & 0.01190                      & \textbf{0.00570}                                            \\
                           & Resnet-110-SD  & 0.12404                          & {\underline{ 0.00762}}                     & 0.01311                            & 0.02120                     & 0.01588                      & \textbf{0.00584}                                            \\
                           & DenseNet-40    & 0.15901                          & {\underline{ 0.00437}}                     & \textbf{0.00368}                   & 0.02205                     & 0.00518                      & 0.01077                                                     \\
                           & Wide Resnet-32 & 0.14078                          & \textbf{0.00414}                  & {\underline{ 0.00548}}                      & 0.01915                     & 0.01099                      & 0.00916                                                     \\
                           & Lenet-5        & 0.14713                          & 0.00787                           & 0.01249                            & \textbf{0.00643}            & 0.02682                      & {\underline{ 0.00752}}                                               \\ \midrule
\multirow{2}{*}{ImageNet}  & Densenet-161   & 0.04266                          & 0.01051                           & \underline{0.00868}                            & 0.03372                     & 0.02536                      & \textbf{0.00278}                                            \\
                           & Resnet-152     & 0.04851                          & 0.01167                           & \underline{0.00776}                            & 0.04093                     & 0.02839                      & \textbf{0.00346}                                            \\ \midrule
SVHN                       & Resnet-152-SD  & 0.00485                          & \underline{0.00388}                           & 0.00410                            & 0.00407                     & \underline{0.00388}                      & \textbf{0.00155}                                            \\ 
\bottomrule
\end{tabular}
    \vspace{1ex}

    \caption{\em KS Error within top-2 prediction (with lowest in bold and second lowest underlined) on various image classification datasets and models with different calibration methods. 
    }
    \label{tab:res_KSE2}  
\end{table}
}
\SKIP{
\begin{table}[t!]
\scriptsize
\begin{tabular}{ll|cccccc}
\toprule
Dataset &
  Model &
  Uncalibrated &
  Temp. Scaling &
  Vector Scaling &
  MS-ODIR &
  Dir-ODIR &
  \textbf{Ours (Spline)} \\ \midrule
\multirow{5}{*}{CIFAR-10}  & Resnet-110     & 0.01805                          & \textbf{0.00097}                  & 0.00176                            & {\underline{ 0.00140}}               & 0.00195                      & 0.00277                                                     \\
                           & Resnet-110-SD  & 0.01423                          & 0.00111                           & 0.00089                            & {\underline{ 0.00082}}               & \textbf{0.00073}             & 0.00104                                                     \\
                           & DenseNet-40    & 0.02256                          & 0.00435                           & 0.00409                            & {\underline{ 0.00395}}               & \textbf{0.00348}             & 0.00571                                                     \\
                           & Wide Resnet-32 & 0.01812                          & 0.00145                           & \textbf{0.00105}                   & {\underline{ 0.00124}}               & 0.00139                      & 0.00537                                                     \\
                           & Lenet-5        & 0.03545                          & 0.00832                           & 0.00831                            & \textbf{0.00631}            & 0.00804                      & {\underline{ 0.00670}}                                               \\ \midrule
\multirow{5}{*}{CIFAR-100} & Resnet-110     & 0.14270                          & 0.00885                           & {\underline{ 0.00649}}                      & 0.01425                     & 0.01190                      & \textbf{0.00503}                                            \\
                           & Resnet-110-SD  & 0.12404                          & {\underline{ 0.00762}}                     & 0.01311                            & 0.02120                     & 0.01588                      & \textbf{0.00684}                                            \\
                           & DenseNet-40    & 0.15901                          & {\underline{ 0.00437}}                     & \textbf{0.00368}                   & 0.02205                     & 0.00518                      & 0.00724                                                     \\
                           & Wide Resnet-32 & 0.14078                          & \textbf{0.00414}                  & {\underline{ 0.00548}}                      & 0.01915                     & 0.01099                      & 0.01017                                                     \\
                           & Lenet-5        & 0.14713                          & 0.00787                           & 0.01249                            & \underline{0.00643}            & 0.02682                      & {\textbf{ 0.00518}}                                               \\ \midrule
\multirow{2}{*}{ImageNet}  & Densenet-161   & 0.04266                          & 0.01051                           & \underline{0.00868}                            & 0.03372                     & 0.02536                      & \textbf{0.00408}                                            \\
                           & Resnet-152     & 0.04851                          & 0.01167                           & \underline{0.00776}                            & 0.04093                     & 0.02839                      & \textbf{0.00247}                                            \\ \midrule
SVHN                       & Resnet-152-SD  & 0.00485                          & \underline{0.00388}                           & 0.00410                            & 0.00407                     & \underline{0.00388}                      & \textbf{0.00158}                                            \\ 
\bottomrule
\end{tabular}
    \vspace{1ex}

    \caption{\em KS Error within top-2 prediction (with lowest in bold and second lowest underlined) on various image classification datasets and models with different calibration methods. Note, for this experiment we use 14 knots for spline fitting.
    }
    \label{tab:res_KSE2}  
\end{table}
}

\SKIP{
\begin{table}[t!]
\scriptsize
\begin{tabular}{ll|cccccc}
\toprule
Dataset &
  Model &
  Uncalibrated &
  Temp. Scaling &
  Vector Scaling &
  MS-ODIR &
  Dir-ODIR &
  \textbf{Ours (Spline)} \\ \midrule
\multirow{5}{*}{CIFAR-10}  & Resnet-110     & 1.805                          & \textbf{0.097}                  & 0.176                            & {\underline{ 0.140}}               & 0.195                      & 0.742                                                     \\
                           & Resnet-110-SD  & 1.423                          & 0.111                           & 0.089                            & {\underline{ 0.082}}               & \textbf{0.073}             & 1.087                                                     \\
                           & DenseNet-40    & 2.256                          & 0.435                           & 0.409                            & {\underline{ 0.395}}               & \textbf{0.348}             & 1.718                                                     \\
                           & Wide Resnet-32 & 1.812                          & 0.145                           & \textbf{0.105}                   & {\underline{ 0.124}}               & 0.139                      & 0.808                                                     \\
                           & Lenet-5        & 3.545                          & 0.832                           & 0.831                            & \textbf{0.631}            & 0.804                      & {\underline{ 0.686}}                                               \\ \midrule
\multirow{5}{*}{CIFAR-100} & Resnet-110     & 14.270                          & 0.885                           & {\underline{ 0.649}}                      & 1.425                     & 1.190                      & \textbf{0.570}                                            \\
                           & Resnet-110-SD  & 12.404                          & {\underline{ 0.762}}                     & 1.311                            & 2.120                     & 1.588                      & \textbf{0.584}                                            \\
                           & DenseNet-40    & 15.901                          & {\underline{ 0.437}}                     & \textbf{0.368}                   & 2.205                     & 0.518                      & 1.077                                                     \\
                           & Wide Resnet-32 & 14.078                          & \textbf{0.414}                  & {\underline{ 0.548}}                      & 1.915                     & 1.099                      & 0.916                                                     \\
                           & Lenet-5        & 14.713                          & 0.787                           & 1.249                            & \textbf{0.643}            & 2.682                      & {\underline{ 0.752}}                                               \\ \midrule
\multirow{2}{*}{ImageNet}  & Densenet-161   & 4.266                          & 1.051                           & \underline{0.868}                            & 3.372                     & 2.536                      & \textbf{0.278}                                            \\
                           & Resnet-152     & 4.851                          & 1.167                           & \underline{0.776}                            & 4.093                     & 2.839                      & \textbf{0.346}                                            \\ \midrule
SVHN                       & Resnet-152-SD  & 0.485                          & \underline{0.388}                           & 0.410                            & 0.407                     & \underline{0.388}                      & \textbf{0.155}                                            \\ 
\bottomrule
\end{tabular}
    \vspace{1ex}

    \caption{\em KS Error within top-2 prediction (with lowest in bold and second lowest underlined) on various image classification datasets and models with different calibration methods. 
    }
    \label{tab:res_KSE2}  
\end{table}
}

\begin{table}[t!]
\scriptsize
\begin{tabular}{ll|cccccc}
\toprule
Dataset &
  Model &
  Uncalibrated &
  Temp. Scaling &
  Vector Scaling &
  MS-ODIR &
  Dir-ODIR &
  \textbf{Ours (Spline)} \\ \midrule
\multirow{5}{*}{CIFAR-10}  & Resnet-110     & 1.805                          & \textbf{0.097}                  & 0.176                            & {\underline{ 0.140}}               & 0.195                      & 0.277                                                     \\
                           & Resnet-110-SD  & 1.423                          & 0.111                           & 0.089                            & {\underline{ 0.082}}               & \textbf{0.073}             & 0.104                                                     \\
                           & DenseNet-40    & 2.256                          & 0.435                           & 0.409                            & {\underline{ 0.395}}               & \textbf{0.348}             & 0.571                                                     \\
                           & Wide Resnet-32 & 1.812                          & 0.145                           & \textbf{0.105}                   & {\underline{ 0.124}}               & 0.139                      & 0.537                                                     \\
                           & Lenet-5        & 3.545                          & 0.832                           & 0.831                            & \textbf{0.631}            & 0.804                      & {\underline{ 0.670}}                                               \\ \midrule
\multirow{5}{*}{CIFAR-100} & Resnet-110     & 14.270                          & 0.885                           & {\underline{ 0.649}}                      & 1.425                     & 1.190                      & \textbf{0.503}                                            \\
                           & Resnet-110-SD  & 12.404                          & {\underline{ 0.762}}                     & 1.311                            & 2.120                     & 1.588                      & \textbf{0.684}                                            \\
                           & DenseNet-40    & 15.901                          & {\underline{ 0.437}}                     & \textbf{0.368}                   & 2.205                     & 0.518                      & 0.724                                                     \\
                           & Wide Resnet-32 & 14.078                          & \textbf{0.414}                  & {\underline{ 0.548}}                      & 1.915                     & 1.099                      & 1.017                                                     \\
                           & Lenet-5        & 14.713                          & 0.787                           & 1.249                            & \underline{0.643}            & 2.682                      & {\textbf{ 0.518}}                                               \\ \midrule
\multirow{2}{*}{ImageNet}  & Densenet-161   & 4.266                          & 1.051                           & \underline{0.868}                            & 3.372                     & 2.536                      & \textbf{0.408}                                            \\
                           & Resnet-152     & 4.851                          & 1.167                           & \underline{0.776}                            & 4.093                     & 2.839                      & \textbf{0.247}                                            \\ \midrule
SVHN                       & Resnet-152-SD  & 0.485                          & \underline{0.388}                           & 0.410                            & 0.407                     & \underline{0.388}                      & \textbf{0.158}                                            \\ 
\bottomrule
\end{tabular}
    \vspace{1ex}

    \caption{\small {\bf Within-top-2 predictions.} \em KS Error (in \%) within-top-2 prediction (with lowest in bold and second lowest underlined) on various image classification datasets and models with different calibration methods. Note, for this experiment we use 14 knots for spline fitting.
    }
    \label{tab:res_KSEwn2}  
    \vspace{-3ex}
\end{table}

\begin{figure}
\centerline{
\includegraphics[width=\textwidth]{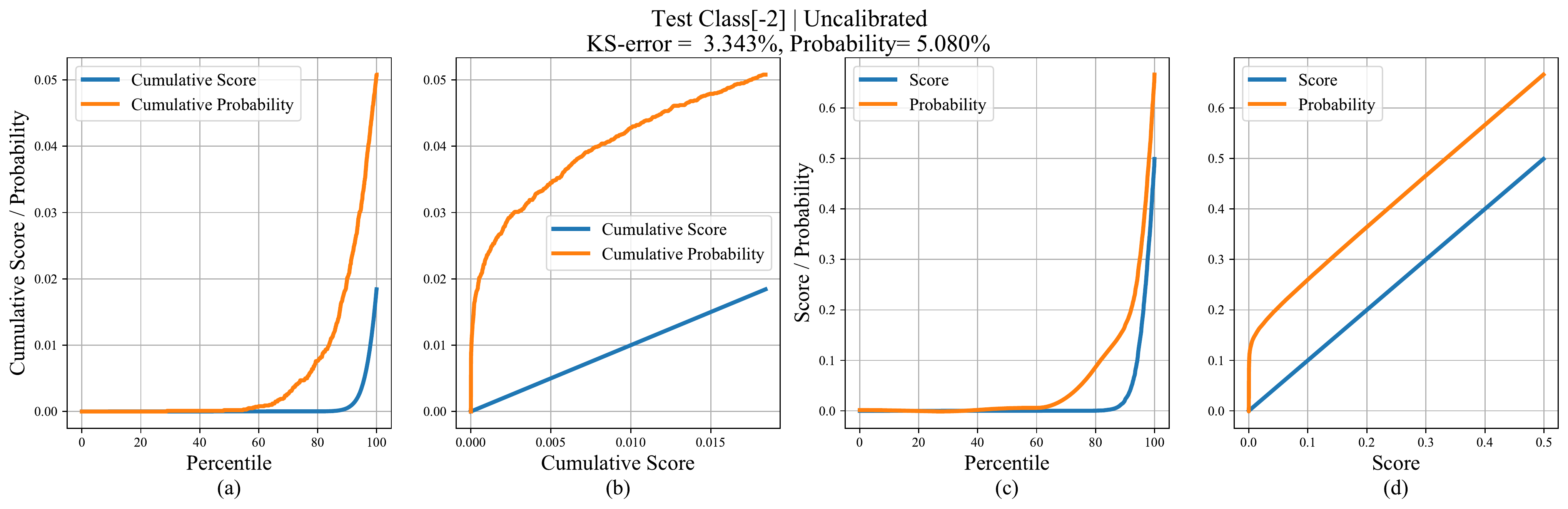}
}
\vspace{-2ex}
\caption{\small {\bf Top-2 predictions, Uncalibrated.} \em Calibration graphs for an uncalibrated DenseNet-40
\citenew{huang2017densely} trained on CIFAR-10 for top-2 class with a KS error of $3.343$\% on the
test set. Here (\textbf{a}) shows the plot of cumulative score and probability versus the
fractile of the test set, (\textbf{b}) shows the same information with the
horizontal axis warped so that the cumulative-score graph is a straight line.  
This is created as scatter plots of cumulative (score, score): blue and (score, probability): orange. 
If the network is perfectly calibrated, the probability line will be a straight line
coincident with the (score, score) line. This shows that the network is substantially
overestimating (score) the probability of the computation. (\textbf{c}) and (\textbf{d}) show plots of (non-cumulative) score and probability plotted against fractile, or score.  How these plots are produced is described in Section 4 of main paper.
}
\label{fig:main-graphs2}
\vspace{-2ex}
\end{figure}

\begin{figure}
\centerline{
\includegraphics[width=\textwidth]{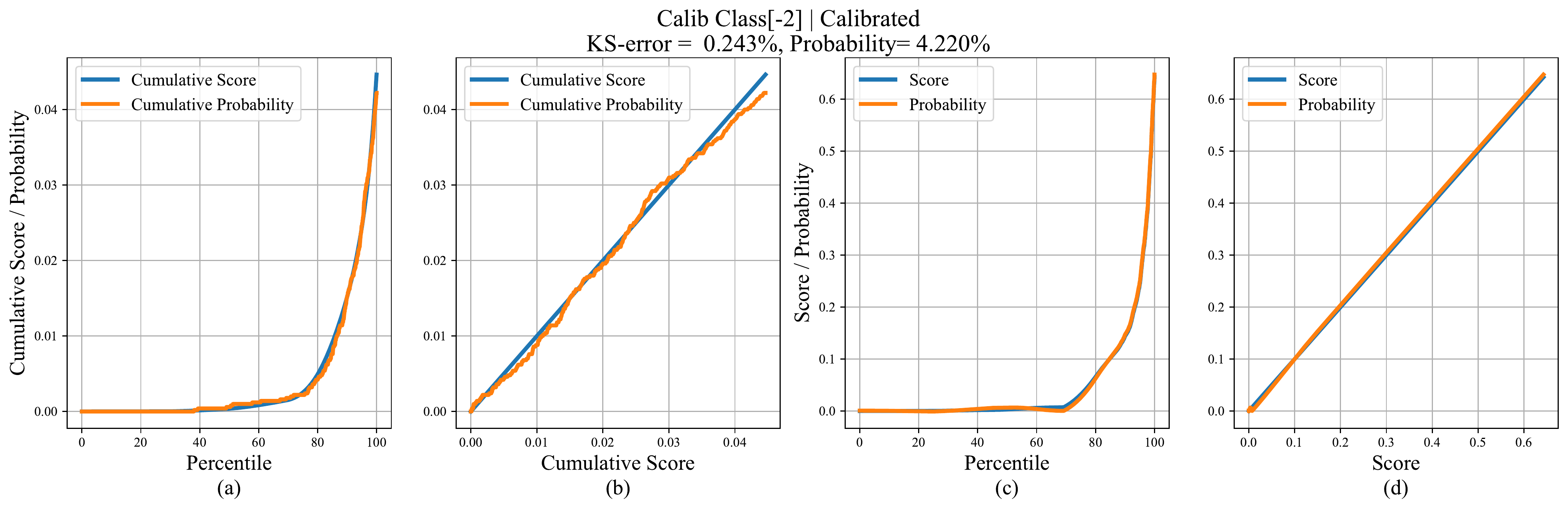} 
}
\vspace{-0.8ex}
\centerline{
\includegraphics[width=\textwidth]{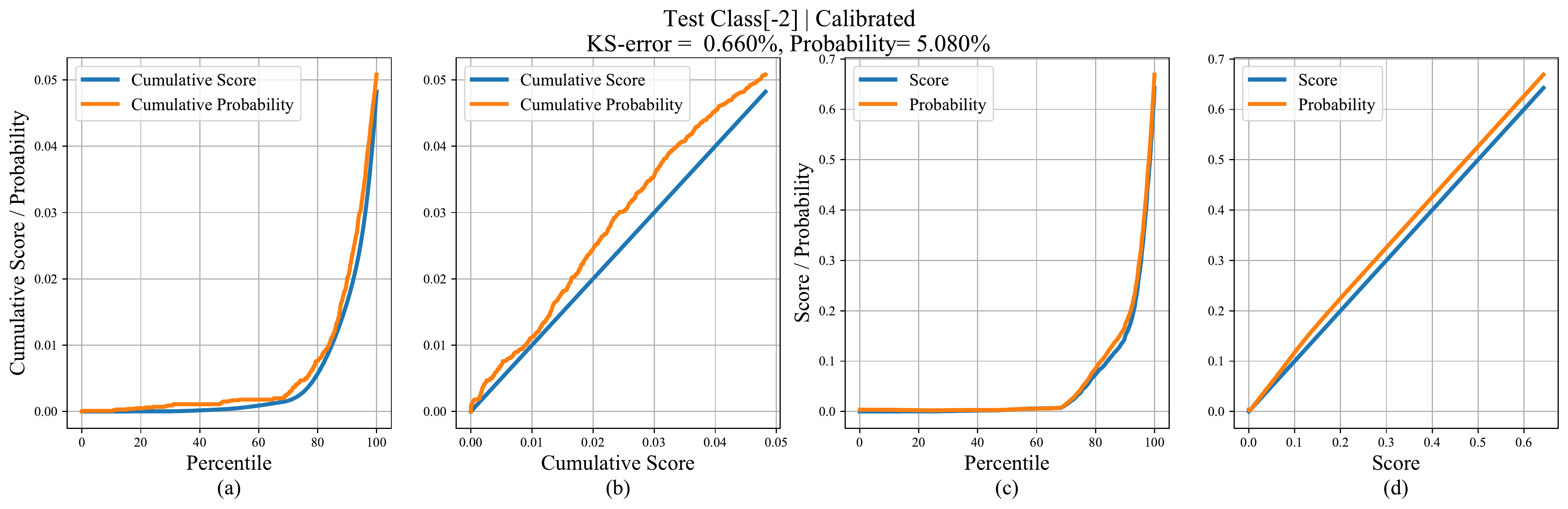} 
}
\vspace{-2ex}
\caption{
\small {\bf Top-2 predictions, Calibrated.} \em The result of the spline calibration method, on the
example given in \fig{main-graphs2} for top-$2$ calibration. 
A recalibration function $\gamma:
\R\rightarrow\R$
is used to adjust the scores, replacing $f_k(\v x)$ with $\gamma(f_k(\v x))$ 
(see Section 4 of main paper).  
As is seen, the network is now almost
perfectly calibrated when tested on the ``calibration'' set (\textbf{top row}) 
used to calibrate it. 
In \textbf{bottom row}, the recalibration function is tested
on a further set ``test''.  It is seen that the result is 
not perfect, but much better than the original results 
in \fig{main-graphs2}d.
}
\label{fig:densenet-calibrated2}
\vspace{-2ex}
\end{figure}

\begin{figure}
\centerline{
\includegraphics[width=\textwidth]{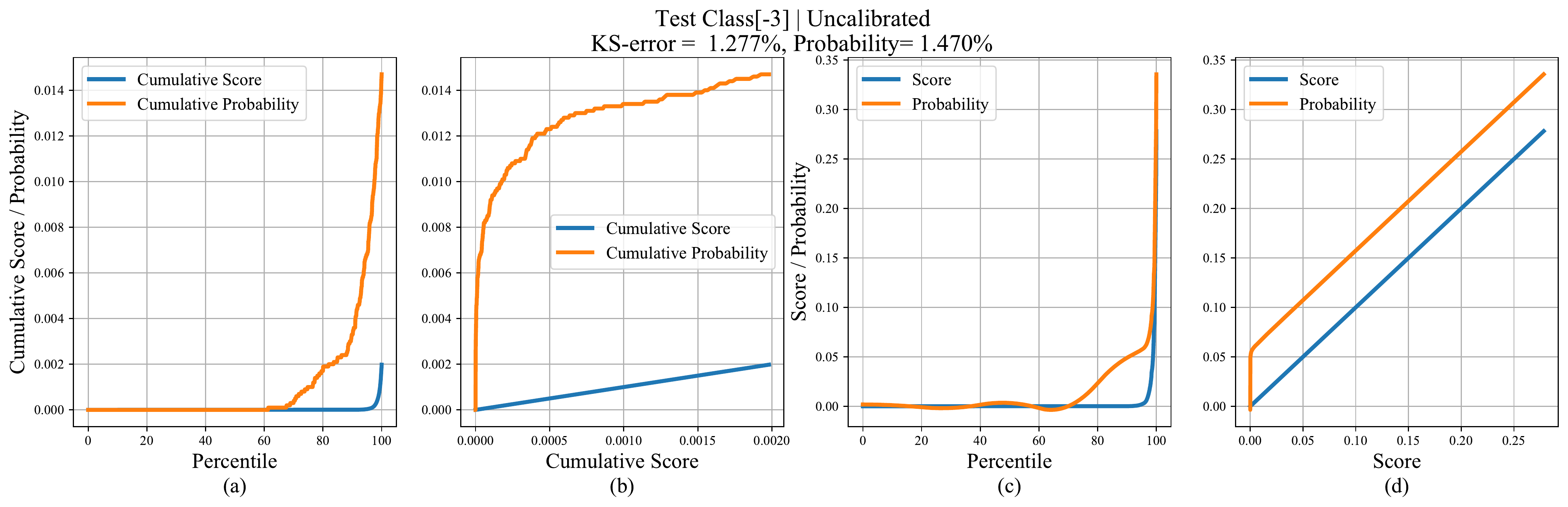}
}
\vspace{-2ex}
\caption{\small {\bf Top-3 predictions, Uncalibrated.} \em Calibration graphs for an uncalibrated DenseNet-40
 trained on CIFAR-10 for top-3 class with a KS error of $1.277$\% on the
test set. Here (\textbf{a}) shows the plot of cumulative score and probability versus the
fractile of the test set, (\textbf{b}) shows the same information with the
horizontal axis warped so that the cumulative-score graph is a straight line.  
(\textbf{c}) and (\textbf{d}) show plots of (non-cumulative) score and probability plotted against fractile, or score.
}
\label{fig:main-graphs3}
\vspace{-2ex}
\end{figure}

\begin{figure}
\centerline{
\includegraphics[width=\textwidth]{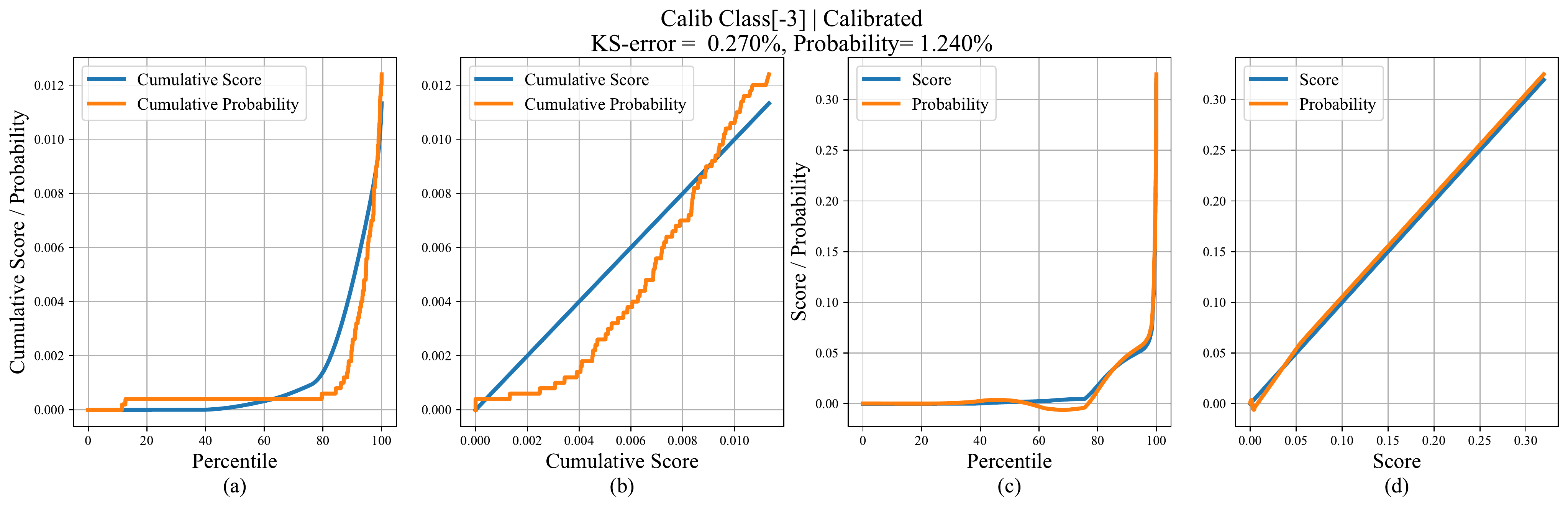} 
}
\vspace{-0.8ex}
\centerline{
\includegraphics[width=\textwidth]{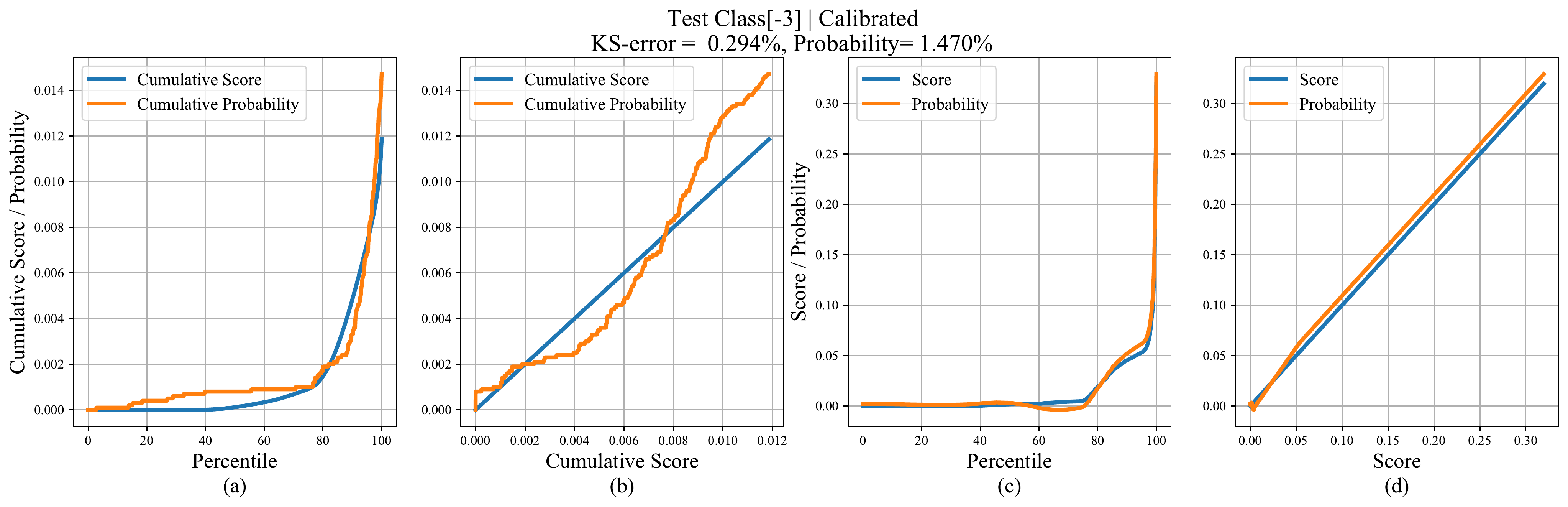} 
}
\vspace{-2ex}
\caption{
\small {\bf Top-3 predictions, Calibrated.} \em The result of the spline calibration method, on the
example given in \fig{main-graphs3} for top-$3$ calibration. 
A recalibration function $\gamma:
\R\rightarrow\R$
is used to adjust the scores, replacing $f_k(\v x)$ with $\gamma(f_k(\v x))$.
As is seen, the network is now almost
perfectly calibrated when tested on the ``calibration'' set (\textbf{top row}) 
used to calibrate it. 
In \textbf{bottom row}, the recalibration function is tested
on a further set ``test''.  It is seen that the result is 
not perfect, but much better than the original results 
in \fig{main-graphs3}d.
}
\label{fig:densenet-calibrated3}
\end{figure}

\begin{figure}
\centerline{
\includegraphics[width=\textwidth]{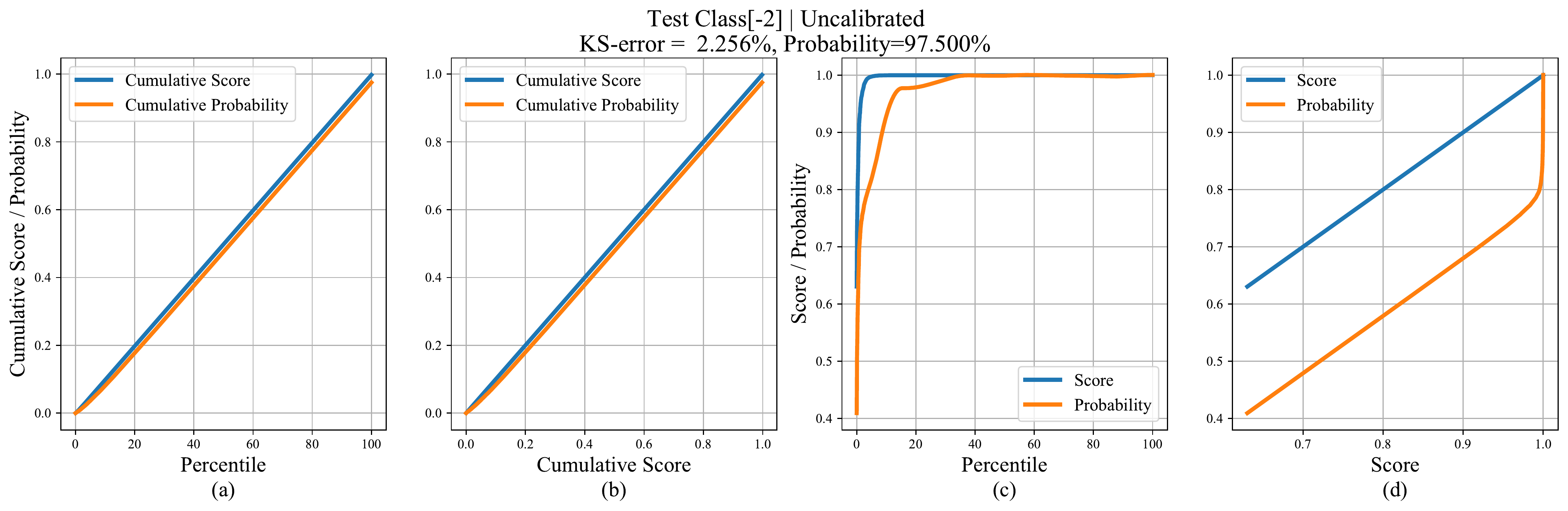}
}
\vspace{-2ex}
\caption{\small {\bf Within-top-2 predictions, Uncalibrated.} \em Calibration graphs for an uncalibrated DenseNet-40
 trained on CIFAR-10 for within-top-2 predictions with a KS error of $2.256$\% on the
test set. Here (\textbf{a}) shows the plot of cumulative score and probability versus the
fractile of the test set, (\textbf{b}) shows the same information with the
horizontal axis warped so that the cumulative-score graph is a straight line.  
(\textbf{c}) and (\textbf{d}) show plots of (non-cumulative) score and probability plotted against fractile, or score.
}
\label{fig:main-graphswn2}
\vspace{-2ex}
\end{figure}

\begin{figure}
\centerline{
\includegraphics[width=\textwidth]{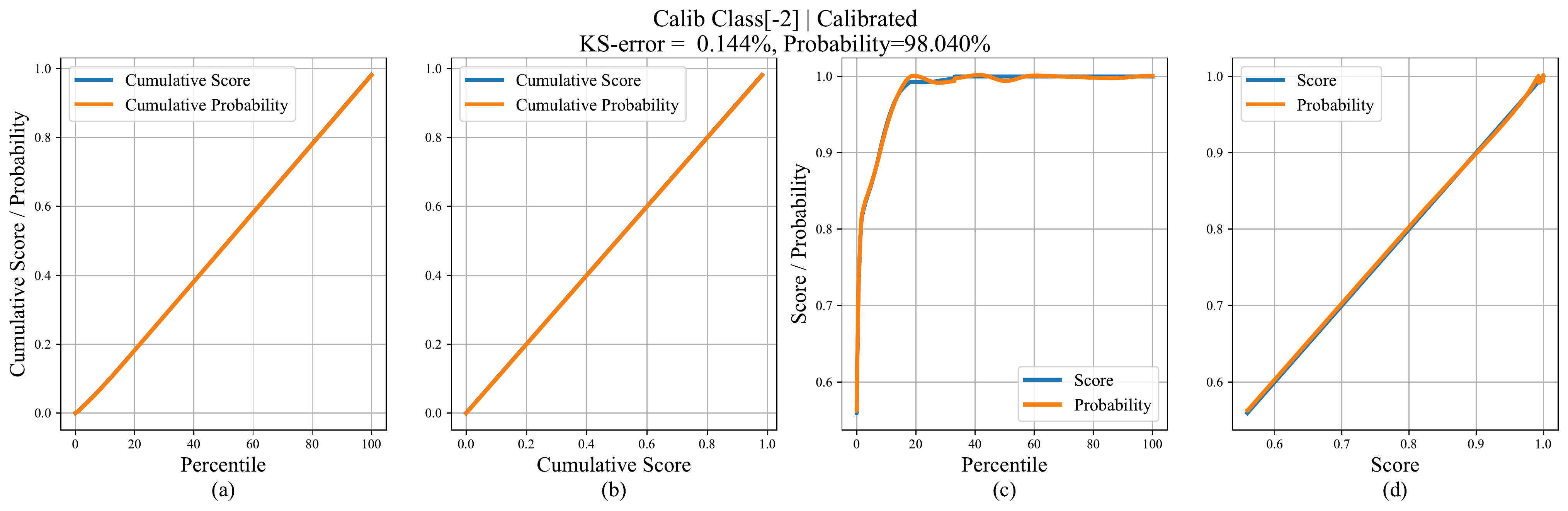} 
}
\vspace{-0.8ex}
\centerline{
\includegraphics[width=\textwidth]{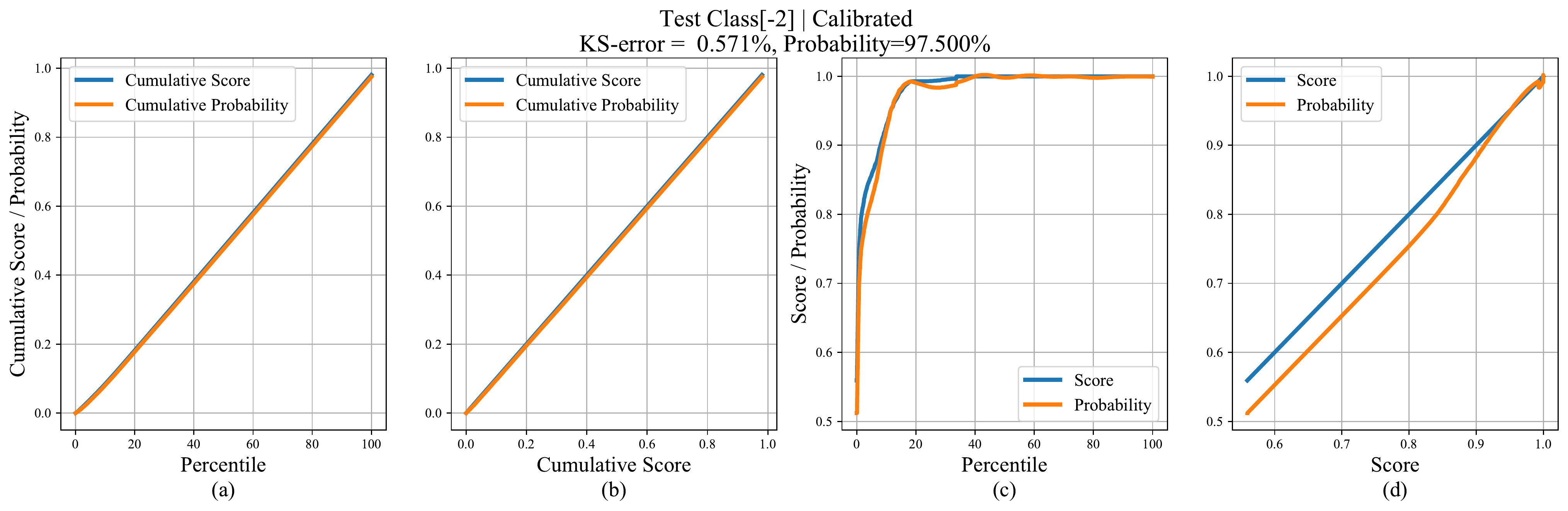} 
}
\vspace{-2ex}
\caption{
\small {\bf Within-top-2 predictions, Calibrated.} \em The result of the spline calibration method, on the
example given in \fig{main-graphswn2} for within-top-$2$ calibration. 
A recalibration function $\gamma:
\R\rightarrow\R$
is used to adjust the scores, replacing $f_k(\v x)$ with $\gamma(f_k(\v x))$.
As is seen, the network is now almost
perfectly calibrated when tested on the ``calibration'' set (\textbf{top row}) 
used to calibrate it. 
In \textbf{bottom row}, the recalibration function is tested
on a further set ``test''.  It is seen that the result is 
not perfect, but much better than the original results 
in \fig{main-graphswn2}d.
}
\label{fig:densenet-calibratedwn2}
\vspace{-2ex}
\end{figure}

\begin{figure}
\centerline{
\includegraphics[width=\textwidth]{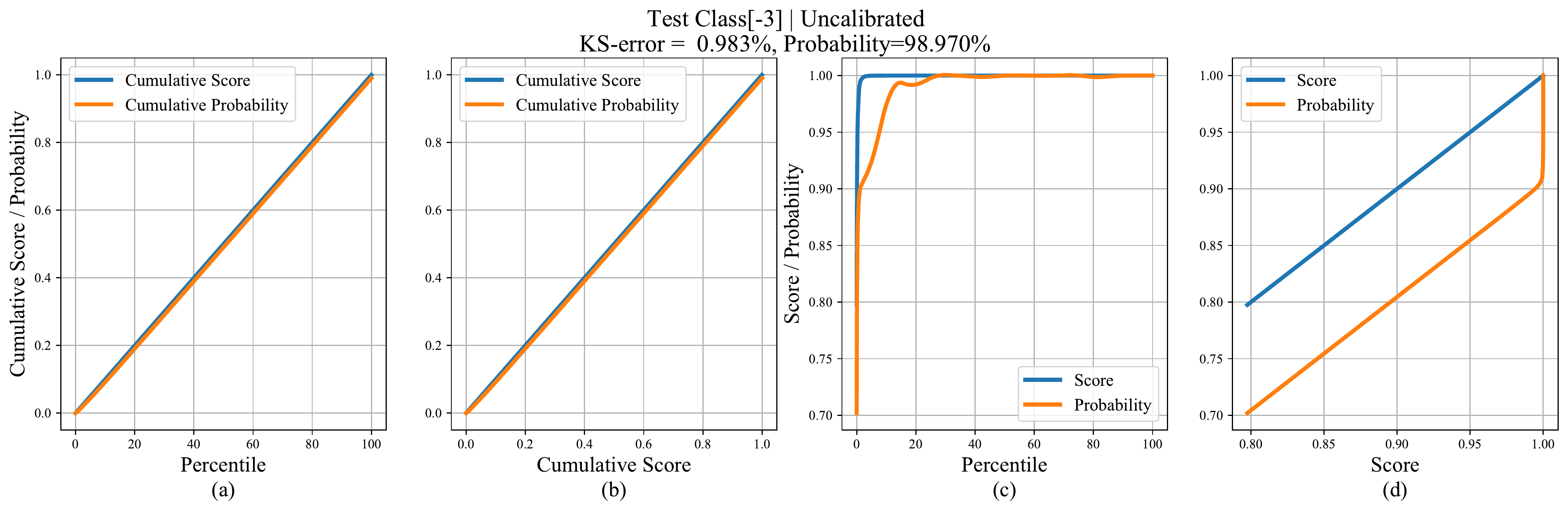}
}
\vspace{-2ex}
\caption{\small {\bf Within-top-3 predictions, Uncalibrated.} \em Calibration graphs for an uncalibrated DenseNet-40
 trained on CIFAR-10 for within-top-3 predictions with a KS error of $0.983$\% on the
test set. Here (\textbf{a}) shows the plot of cumulative score and probability versus the
fractile of the test set, (\textbf{b}) shows the same information with the
horizontal axis warped so that the cumulative-score graph is a straight line.  
(\textbf{c}) and (\textbf{d}) show plots of (non-cumulative) score and probability plotted against fractile, or score.
}
\label{fig:main-graphswn3}
\vspace{-2ex}
\end{figure}

\begin{figure}
\centerline{
\includegraphics[width=\textwidth]{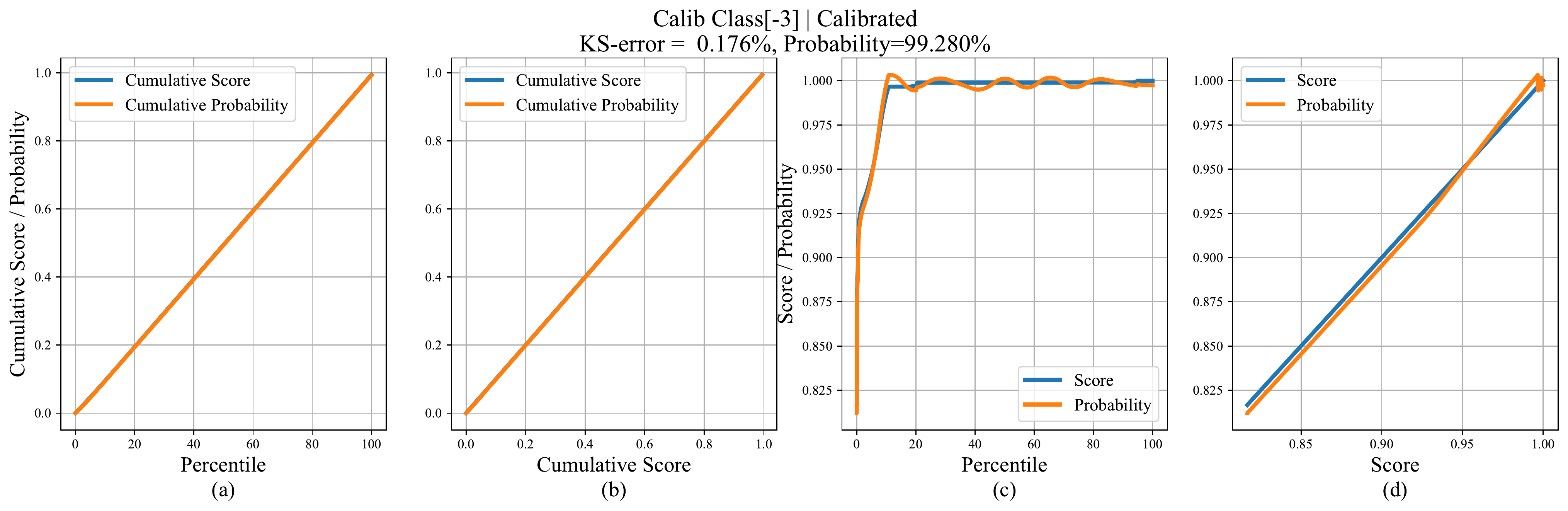} 
}
\vspace{-0.8ex}
\centerline{
\includegraphics[width=\textwidth]{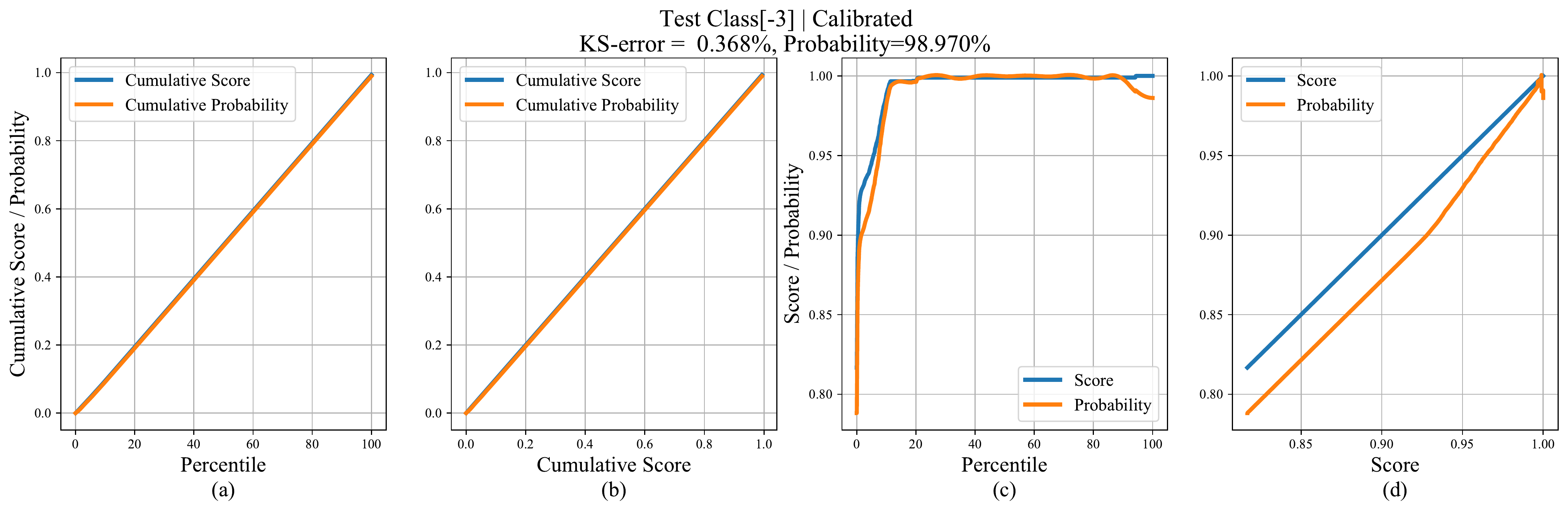} 
}
\vspace{-2ex}
\caption{
\small {\bf Within-top-3 predictions, Calibrated.} \em The result of the spline calibration method, on the
example given in \fig{main-graphswn3} for within-top-$3$ calibration. 
A recalibration function $\gamma:
\R\rightarrow\R$
is used to adjust the scores, replacing $f_k(\v x)$ with $\gamma(f_k(\v x))$.
As is seen, the network is now almost
perfectly calibrated when tested on the ``calibration'' set (\textbf{top row}) 
used to calibrate it. 
In \textbf{bottom row}, the recalibration function is tested
on a further set ``test''.  It is seen that the result is 
not perfect, but much better than the original results 
in \fig{main-graphswn3}d.
}
\label{fig:densenet-calibratedwn3}
\end{figure}

We also provide comparisons of our method against baseline methods for within-top-2 predictions (equation \myref{5} of the main paper) in \tabref{tab:res_KSEwn2} using KS error. Our method achieves comparable or better results for within-top-2 predictions. It should be noted that the scores for top-3 ($f^{(-3)}(\v x)$) or even top-4, top-5, etc., are very close to zero for majority of the samples (due to overconfidence of top-$1$ predictions).
Therefore the calibration error for top-$r$ with $r>2$ predictions is very close to zero and comparing different methods with respect to it is of little value.
Furthermore, for visual illustration, we provide calibration graphs of top-2 predictions in \fig{main-graphs2} and \fig{densenet-calibrated2} for uncalibrated and calibrated network respectively. Similar graphs for top-3, within-top-2, and within-top-3 predictions
are presented in figures~\ref{fig:main-graphs3}~--~\ref{fig:densenet-calibratedwn3}.

We also provide classification accuracy comparisons for different post-hoc calibration methods against our method if we apply calibration for all top-$1,2,3,\ldots,K$ predictions for $K$-class classification problem in \tabref{tab:acc}. We would like to point out that there is negligible change in accuracy between the calibrated networks (using our method) and the uncalibrated ones.

For the sake of completeness, we present calibration results using the existing calibration metric, Expected Calibration Error (ECE)~\citenew{naeini2015obtaining} in \tabref{tab:res_ECE25}. We would like to reiterate the fact that ECE metric is highly dependent on the chosen number of bins and thus does not really reflect true calibration performance. To reflect the efficacy of our proposed calibration method, we also present calibration results using other calibration metrics such as recently proposed binning free measure KDE-ECE~\citenew{zhang2020mix}, MCE (Maximum Calibration Error)~\citenew{guo2017calibration} and Brier Scores for top-1 predictions on ImageNet dataset in \tabref{tab:res_othermetric_imagenet}. Since, the original formulation of Brier Score for multi-class predictions is highly biased on the accuracy and is approximately similar for all calibration methods, we hereby use top-1 Brier Score which is the mean squared error between top-1 scores and ground truths for the top-1 predictions (1 if the prediction is correct and 0 otherwise). It can be clearly observed that our approach consistently outperforms all the baselines on different calibration measures.

\begin{table}
\scriptsize
\begin{tabular}{ll|cccccc}
\toprule
Dataset                    & Model          & Uncalibrated & Temp. Scaling & Vector Scaling & MS-ODIR & Dir-ODIR & \textbf{Ours (Spline)} \\
\midrule
\multirow{5}{*}{CIFAR-10}  & Resnet-110     & \textbf{93.56}        & \textbf{93.56}         & 93.50          & 93.53   & 93.52    & \underline{93.55}             \\
                           & Resnet-110-SD  & 94.04        & 94.04         & 94.04          & \underline{94.18}   & \textbf{94.20}    & 94.05             \\
                           & DenseNet-40    & 92.42        & 92.42         & \underline{92.50}          & \textbf{92.52}   & 92.47    & 92.31             \\
                           & Wide Resnet-32 & 93.93        & 93.93         & \underline{94.21}          & \textbf{94.22}   & \textbf{94.22}    & 93.76             \\
                           & Lenet-5        & 72.74        & 72.74         & \underline{74.48}          & 74.44   & \textbf{74.52}    & 72.64             \\ \midrule
\multirow{5}{*}{CIFAR-100} & Resnet-110     & 71.48        & 71.48         & \underline{71.58}          & 71.55   & \textbf{71.62}    & 71.50            \\
                           & Resnet-110-SD  & 72.83        & 72.83         & \textbf{73.60}          & \underline{73.53}   & 73.14    & 72.81             \\
                           & DenseNet-40    & 70.00        & 70.00         & 70.13          & \textbf{70.40}   & \underline{70.24}    & 70.17             \\
                           & Wide Resnet-32 & 73.82        & 73.82         & 73.87          & \textbf{74.05}   & \underline{73.99}    & 73.74             \\
                           & Lenet-5        & 33.59        & 33.59         & 36.42          & \textbf{37.58}   & \underline{37.52}    & 33.55             \\ \midrule
\multirow{2}{*}{ImageNet}  & Densenet-161   & 77.05        & 77.05         & 76.72          & \underline{77.15}   & \textbf{77.19}    & 77.05             \\
                           & Resnet-152     & \underline{76.20}        & \underline{76.20}         & 75.87          & 76.12   & \textbf{76.24}    & 76.07             \\ \midrule
SVHN                       & Resnet-152-SD  & 98.15        & 98.15         & 98.13          & 98.12   & \textbf{98.19}    & \underline{98.17}            \\ \bottomrule
\end{tabular}
    \vspace{1ex}

\caption{\small \em Classification (top-1) accuracy (with highest in bold and second highest underlined) post calibration on various image classification datasets and models with different calibration methods. Note, only a negligible change in accuracy is observed in our method compared to the uncalibrated networks.}
\label{tab:acc}
\end{table}
\SKIP{
\begin{table}
\scriptsize
\begin{tabular}{ll|cccccc}
\toprule
Dataset                    & Model          & Uncalibrated & Temp. Scaling & Vector Scaling & MS-ODIR & Dir-ODIR & \textbf{Ours (Spline)} \\
\midrule
\multirow{5}{*}{CIFAR-10}  & Resnet-110     & 0.04750      & 0.01224       & 0.01092        & 0.01276 & 0.01240  & 0.01281             \\
                           & Resnet-110-SD  & 0.04135      & 0.00777       & 0.00752        & 0.00684 & 0.00859  & 0.01076             \\
                           & DenseNet-40    & 0.05507      & 0.01006       & 0.01207        & 0.01250 & 0.01268  & 0.01181             \\
                           & Wide Resnet-32 & 0.04512      & 0.00905       & 0.00852        & 0.00941 & 0.00965  & 0.01189             \\
                           & Lenet-5        & 0.05188      & 0.01999       & 0.01462        & 0.01504 & 0.01300  & 0.01595             \\
\midrule
\multirow{5}{*}{CIFAR-100} & Resnet-110     & 0.18480      & 0.02428       & 0.02722        & 0.03011 & 0.02806  & 0.01964             \\
                           & Resnet-110-SD  & 0.15861      & 0.01335       & 0.02067        & 0.02277 & 0.02046  & 0.01415             \\
                           & DenseNet-40    & 0.21159      & 0.01255       & 0.01598        & 0.02855 & 0.01410  & 0.01828             \\
                           & Wide Resnet-32 & 0.18784      & 0.01667       & 0.01785        & 0.02870 & 0.02128  & 0.02322             \\
                           & Lenet-5        & 0.12117      & 0.01535       & 0.01350        & 0.01696 & 0.02159  & 0.01912             \\
\midrule
\multirow{2}{*}{ImageNet}  & Densenet-161   & 0.05720      & 0.02059       & 0.02637        & 0.04337 & 0.03989  & 0.00898             \\
                           & Resnet-152     & 0.06545      & 0.02166       & 0.02641        & 0.05377 & 0.04556  & 0.01200             \\
\midrule
SVHN                       & Resnet-152-SD  & 0.00877      & 0.00675       & 0.00630        & 0.00646 & 0.00651  & 0.00973            \\
\bottomrule
\end{tabular}
    \vspace{1ex}

\caption{\em Expected Calibration Error (in \%) using 25 bins on various image classification datasets and models with different calibration methods.}
\end{table}
}

\begin{table}
\scriptsize
\begin{tabular}{ll|cccccc}
\toprule
Dataset                    & Model          & Uncalibrated & Temp. Scaling & Vector Scaling & MS-ODIR & Dir-ODIR & \textbf{Ours (Spline)} \\
\midrule
\multirow{5}{*}{CIFAR-10}  & Resnet-110     & 4.750      & 1.224       & \underline{1.092}        & 1.276 & 1.240  & \textbf{1.011}             \\
                           & Resnet-110-SD  & 4.135      & 0.777       & \underline{0.752}        & \textbf{0.684} & 0.859  & 0.992             \\
                           & DenseNet-40    & 5.507      & \textbf{1.006}       & \underline{1.207}        & 1.250 & 1.268  & 1.389             \\
                           & Wide Resnet-32 & 4.512      & \underline{0.905}       & \textbf{0.852}        & 0.941 & 0.965  & 1.003             \\
                           & Lenet-5        & 5.188      & 1.999       & 1.462        & 1.504 & \textbf{1.300}  & \underline{1.333}             \\
\midrule
\multirow{5}{*}{CIFAR-100} & Resnet-110     & 18.480      & \underline{2.428}       & 2.722        & 3.011 & 2.806  & \textbf{1.868}             \\
                           & Resnet-110-SD  & 15.861      & \textbf{1.335}       & 2.067        & 2.277 & 2.046  & \underline{1.766}             \\
                           & DenseNet-40    & 21.159      & \textbf{1.255}       & 1.598        & 2.855 & \underline{1.410}  & 2.114             \\
                           & Wide Resnet-32 & 18.784      & \textbf{1.667}       & 1.785        & 2.870 & 2.128  & \underline{1.672}             \\
                           & Lenet-5        & 12.117      & 1.535       & \underline{1.350}        & 1.696 & 2.159  & \textbf{1.029}             \\
\midrule
\multirow{2}{*}{ImageNet}  & Densenet-161   & 5.720      & \underline{2.059}       & 2.637        & 4.337 & 3.989  & \textbf{0.798}             \\
                           & Resnet-152     & 6.545      & \underline{2.166}       & 2.641        & 5.377 & 4.556  & \textbf{0.913}             \\
\midrule
SVHN                       & Resnet-152-SD  & 0.877      & 0.675       & \textbf{0.630}        & \underline{0.646} & 0.651  & 0.832            \\
\bottomrule
\end{tabular}
    \vspace{1ex}

\caption{\small \em ECE for top-1 predictions (in \%) using 25 bins (with lowest in bold and second lowest underlined) on various image classification datasets and models with different calibration methods. Note, for this experiment we use 13 knots for spline fitting.}
    \label{tab:res_ECE25}  

\end{table}
\begin{table}[ht]
\centering
\scriptsize
\begin{tabular}{cl|ccccc}
\toprule
Calibration Metric     & Model          & Uncalibrated & Temp. Scaling & MS-ODIR & Dir-ODIR & \textbf{Ours (Spline)} \\
\midrule
\multirow{2}{*}{\textbf{KDE-ECE}}  & Densenet-161 &  0.03786 & \underline{0.01501} & 0.02874	& 0.02979	& \textbf{0.00637}             \\
& Resnet-152     & 	0.04650	& \underline{0.01864} & 0.03448 &	0.03488	& \textbf{0.00847}           \\
\midrule
\multirow{2}{*}{\textbf{MCE}}  & Densenet-161   & 0.13123 &	\textbf{0.05442} & 0.09077	& 0.09653	&  \underline{0.06289}             \\
& Resnet-152     & 0.15930	& \underline{0.09051} & 0.11201	& 0.09868	& \textbf{0.04950} \\    
\midrule
\multirow{2}{*}{\textbf{Brier Score}}  & Densenet-161   & 0.12172 &	\underline{0.11852} &	0.11982	& 0.11978	& \textbf{0.11734}             \\
& Resnet-152     & 0.12626 &	\underline{0.12145} &	0.12406 &	0.12308 &	\textbf{0.12034}          \\
\bottomrule
\end{tabular}
    \vspace{1ex}

\caption{\small \em Calibration Error using other different metrics such as binning-free KDE-ECE~\citenew{zhang2020mix}, MCE (Maximum Calibration Error)~\citenew{guo2017calibration} and Brier Score for top-1 predictions (with lowest in bold and and second lowest underlined) on ImageNet dataset with different calibration methods. Note, for this experiment we use 6 knots for spline fitting.}
    \label{tab:res_othermetric_imagenet}  

\end{table}

\else
\begin{abstract}
  Calibrating neural networks is of utmost importance when employing them in safety-critical applications where the downstream decision making depends on the predicted probabilities.
  Measuring calibration error amounts to comparing two empirical distributions.
  In this work, we introduce a {\em binning-free calibration measure} inspired by the classical Kolmogorov-Smirnov (KS) statistical test in which the main idea is to compare the respective cumulative probability distributions.
  From this, by approximating the empirical cumulative distribution using a differentiable function via splines, we obtain a {\em recalibration function}, which maps the network outputs to actual (calibrated) class assignment probabilities. 
  The spline-fitting is performed using a held-out calibration set and the obtained recalibration function is evaluated on an unseen test set.
  We tested our method against existing calibration approaches on various image classification datasets and our spline-based recalibration approach consistently outperforms existing methods on KS error as well as other commonly used calibration measures.
\end{abstract}
\section{Introduction}
Despite the success of modern neural networks they are shown to be poorly
calibrated~\citenew{guo2017calibration}, which has led to a growing interest in the
calibration of neural networks over the past few
years~\citenew{kull2019beyond,kumar2019verified,kumar2018trainable,muller2019does}.
Considering classification problems, a classifier is said to be {\em calibrated}
if the probability values it associates with the class labels match the true
probabilities of correct class assignments.
For instance, if an image classifier outputs 0.2 probability for the ``horse''
label for 100 test images, then out of those 100 images approximately 20 images
should be classified as horse.
It is important to ensure calibration when using classifiers for
safety-critical applications such as medical image analysis and autonomous
driving where the downstream decision making depends on the predicted
probabilities.

One of the important aspects of machine learning research is the measure used to
evaluate the performance of a model and in the context of calibration, this
amounts to measuring the difference between two empirical probability
distributions.
To this end, the popular metric, Expected Calibration Error
(ECE)~\citenew{naeini2015obtaining}, approximates the classwise probability
distributions using histograms and takes an expected difference.
This histogram approximation has a weakness that the resulting calibration error
depends on the binning scheme (number of bins and bin divisions).
Even though the drawbacks of ECE have been pointed out and some improvements have
been proposed~\citenew{kumar2019verified,nixon2019measuring}, the histogram approximation has not been eliminated.%
\footnote{We 
consider metrics that measure classwise (top-$r$) calibration
error~\citenew{kull2019beyond}. Refer to \sect{prelim} 
for details.}

In this paper, we first introduce a simple, {\em binning-free calibration
measure} inspired by the classical Kolmogorov-Smirnov (KS) statistical
test~\citenew{kolmogorov1933sulla,smirnov1939estimation}, 
which also provides an effective visualization of
the degree of miscalibration similar to the reliability
diagram~\citenew{niculescu2005predicting}.
To this end, the main idea of the KS-test is to compare the respective classwise
cumulative (empirical) distributions.
Furthermore, by approximating the empirical cumulative distribution using a
differentiable function via splines~\citenew{Mckinley_cubicspline}, we obtain an {\em
analytical recalibration function}\footnote{Open-source implementation available at \url{https://github.com/kartikgupta-at-anu/spline-calibration}} which maps the given network outputs to the
actual class assignment probabilities.
Such a direct mapping was previously unavailable and the problem has been
approached indirectly via learning, for example, by optimizing the
(modified) cross-entropy
loss~\citenew{guo2017calibration,mukhoti2020calibrating,muller2019does}.
Similar to the existing methods~\citenew{guo2017calibration,kull2019beyond} the
spline-fitting is performed using a held-out calibration set and the obtained
recalibration function is evaluated on an unseen test set.

We evaluated our method against existing calibration approaches on various image
classification datasets and our spline-based recalibration approach consistently
outperforms existing methods on KS error, ECE as well as other commonly used
calibration measures.
Our approach to calibration does not update the
model parameters, which allows it to be applied on any trained network and it
retains the original classification accuracy in all the tested cases.


\section{Notation and Preliminaries}\label{sec:prelim}
We abstract the network as a function $f_\theta: \calD \rightarrow  \out^K$,
where $\calD\subset\R^d$,
and write $f_\theta(\v x) = \v z$.  Here, $\v x$ may be an image,
or other input datum, and $\v z$ is a vector, sometimes known as the
vector of {\em logits}.   In this paper, the parameters $\theta$ 
will not be considered,
and we write simply $f$ to represent the network function.
We often refer to this function as a {\em classifier}, and in theory this could be of some other type than a neural network.

In a classification problem, $K$ is the number of 
classes to be distinguished, and we call the value $z_k$ (the $k$-th
component
of vector $\v z$) the 
{\em score} for the class $k$.  If the final layer of a network is a
{\em softmax}
layer, then the values $z_k$ satisfy $\sum_{k=1}^K z_k = 1$, and $z_k \ge 0$.
Hence, the $z_k$ are pseudo-probabilities, though they do not necessarily
have anything to do with real probabilities of correct class assignments. 
Typically,
the value $y^* = \arg\max_k z_k$ is taken as the ({\em top-$1$}) prediction of the
network,
and the corresponding score, $\max_k z_k$ is called the {\em confidence} of the
prediction.  However,
the term confidence does not have any mathematical meaning
in this context
and we deprecate its use.

We assume we are given a set of training data $(\v x_i, y_i)_{i=1}^n$,
where $\v x_i \in \calD$ is an input data element, which for simplicity we call an
image,  
and $y_i \in \calK = \{1, \ldots, K\}$ is
the so-called ground-truth label.%
\SKIP{
\footnote{
It is assumed here that each sampled image has a single 
supplied ground-truth label, though the same image
may potentially be supplied with a different label in a different 
sample.  The problem of
calibrating a network in which each input image may have multiple labels
is interesting, and a subject for further analysis.
}
} 
Our method also uses two other sets of data, called {\em calibration data} and
{\em test data}.

It would be desirable if the numbers $z_k$ output by a network represented true
probabilities.
For this to make sense, we posit the existence of joint random variables $(X,
Y)$,
where $X$ takes values in a domain $\calD \subset \R^d$, and $Y$ takes values
in
$\calK$.  Further, let $Z = f(X)$, another random variable, and 
$Z_k =f_k(X)$ be its
$k$-th component.
Note that in this formulation $X$ and $Y$ are joint random variables, and
the probability $P(Y ~|~ X)$ is {\em not} assumed to be $1$ for single class,
and $0$ for the others. 

A network is said to be {\em calibrated} if for every class $k$, 
\begin{equation}
\label{eq:calibration-definition}
P(Y = k ~|~ Z = \v z) =z_k ~.
\end{equation}
This can be written briefly as $P(k ~|~ f(\v x)) = f_k(\v x) = z_k$.
Thus, if the network
takes input $\v x$ and outputs $\v z = f(\v x)$, then $z_k$ represents
the
probability (given $f(\v x)$) that image $\v x$ belongs to class
$k$.

The probability $P(k ~|~ \v z)$ is difficult to evaluate, even empirically,
and most metrics (such as ECE) use or measure a different notion  
called {\em classwise calibration}~\citenew{kull2019beyond, zadrozny2002transforming}, defined as,
\begin{equation}
\label{eq:classwise-calibration-definition}
P(Y = k ~|~ Z_k = z_k) =z_k ~.
\end{equation}
This paper uses this definition \eq{classwise-calibration-definition} of calibration in the proposed KS metric.

\SKIP{
For instance, given a bunch of meteorological readings $\v x$, a network
may make 
a prediction that it will rain tomorrow, so there is no unique answer.  The
network can have two outputs, ($1$ = ``rain'',
$2$ = ``no rain'').  Then it is desirable that the numbers $\v z = f(\v x)$
produced
by the network satisfy
\newcommand{\rain}{{\text{``rain''}}}
\[
z_\rain = P(\rain ~|~ f(\v x)) ~,
\]
for then one may use the output $f(\v x)$ of the network as a rational basis for
a decision to
carry an umbrella or not.
}
Calibration and accuracy of a network are different concepts.  For instance,
one may consider a classifier that simply outputs the class probabilities for
the data, ignoring the input $\v x$.  Thus, if $f_k(\v x) = z_k = P(Y = k)$,
this classifier $f$ is calibrated but the accuracy is no better than the random predictor.  Therefore, in calibration of a classifier, it is important that
this is not done while sacrificing classification (for instance top-$1$) accuracy.

\SKIP{ 
\paragraph{Perfect prediction. }
\rih{This can go if we need space}
The definition of calibration requires that $z_k = P(Y = k ~|~ Z = \v z)$ 
(Eq.~\eqref{eq:calibration-definition}), 
where $\v z$ is the vector of class probabilities, conditional
on the input data $\v x$.
However, one may ask why one does not require that
\begin{equation}
    z_k = f_k(\v x) = P(Y = k ~|~ X = \v x) ~.
\end{equation}
This would indeed be desirable, but
in this case, it is more than can be reasonably attained.  With
this definition, only a perfect network can be calibrated, one that
can accurately output the class probabilities for every input
datum.  For many applications (weather prediction, medical
diagnosis or prognosis prediction for instance) this goal is 
probably unattainable.   

Typically $f$ (represented by the network) is a many-to-one mapping and the input
dimension is much larger than the number of classes.
Thus, we require less. 
It is sufficient that when a network
says ``this image is a horse'' with score $0.75$,  then
$75\%$ of the time,  the image is, in fact, a horse.
This is what the definition of calibration seeks to express.
} 

\paragraph{The top-$r$ prediction. }

The classifier $f$ being calibrated means that 
$f_k(\v x)$ is calibrated for each class $k$, not only for the top class.  
This means that scores
$z_k$ for all classes $k$ give a meaningful estimate of
the probability of the sample belonging to class $k$.  
This is particularly important in medical diagnosis where one
may wish to have a reliable estimate of the probability
of certain unlikely diagnoses.  

Frequently, however, one is most interested in 
the probability of the top scoring class, the top-$1$ 
prediction, or in general the top-$r$ prediction.
%
%
Suppose a classifier $f$ is given with values in $\out^K$
and let $y$ be the ground truth label.
Let us use $f^{(-r)}$ to denote the $r$-th top score (so $f^{(-1)}$ 
would denote the top score; the notation follows python semantics in which 
$A[-1]$ represents the last element in array $A$).
Similarly we define $\max^{(-r)}$ for the $r$-th largest value.
Let $f^{(-r)}: \calD \rightarrow \out$ be defined as
\begin{equation}\label{eq:top-1-pred}
f^{(-r)} (\v x) = \text{max}^{(-r)}_k f_k(\v x) ~,\quad\text{and}\quad
y^{(-r)} = 
 \begin{cases} 
      1 & \mbox{ if }  y = \arg\max_k^{(-r)} f_k(\v x)  \\
      0 & \mbox{ otherwise ~}. 
   \end{cases}
\end{equation}
In words, $y^{(-r)}$ is $1$ if the $r$-th top predicted class is the correct
(ground-truth) choice.
The network is calibrated for the top-$r$ predictor if for all scores $\sigma$,
\begin{equation}
\label{eq:top-r-calibrated}
P(y^{(-r)} = 1 ~|~ f^{(-r)}(\v x) = \sigma) = \sigma  ~.
\end{equation}
In words, the conditional probability that the top-$r$-th choice of the network is
the correct choice, is equal
to the $r$-th top score.  

Similarly, one may consider probabilities that a datum belongs to one of the top-$r$ scoring classes.
The classifier is calibrated for being within-the-top-$r$ classes if
\begin{equation}
\label{eq:within-top-r-calibrated}
P\big(\textstyle\sum_{s=1}^r y^{(-s)} = 1 ~\big|~ \sum_{s=1}^r f^{(-s)}(\v x) = \sigma\big) = \sigma  ~.
\end{equation}
Here, the sum on the left is $1$ if the ground-truth label is among the
top $r$ choices, $0$ otherwise, and the sum on the right is the sum
of the top $r$ scores.
 


\section{Kolmogorov-Smirnov Calibration Error}
We now consider a way to measure if a classifier is classwise calibrated, including
top-$r$ and within-top-$r$ calibration.
This test
is closely related to the Kolmogorov-Smirnov test~\citenew{kolmogorov1933sulla,smirnov1939estimation} for the
equality of two
probability distributions.  This may be applied when the probability
distributions are represented by samples.

We start with the definition of classwise calibration:
\begin{align}
P(Y = k ~|~ f_k(X) = z_k) &= z_k ~.\\\nonumber
P(Y = k, ~f_k(X) = z_k) &= z_k \,  P(f_k(X) = z_k)~, \quad\mbox{Bayes' rule}~.
\end{align}
This may be written more simply but with a less precise notation as
%
\[P(z_k, ~k) = z_k \,  P(z_k) ~.\]

\SKIP{
\rih{Something funny here. }The left-hand side of this equation, $P(z_k ~|~ k) \, P(k)$,  
is the distribution of scores $z_k = f_k(\v x)$ normalized by the
class probability
for class $k$. This may be thought of as a continuous histogram
of scores for class $k$.   The expression $P(z_k)$ on the
right is the distribution of scores $z_k = f_k(\v x)$ for all classes, which is
multiplied
by $z_k$.
This can also be written as
\begin{equation}
\label{eq:KS-distributions-2}
P(z_k ~|~ k) = z_k \,  P(z_k) / P(k) ~,
\end{equation}
where now, the left hand side is a probability distribution.
}

\paragraph{Motivation of the KS test. }
One is motivated to test the equality (or difference between) two
distributions,
defined on the interval $[0,1]$.  However, instead of having a functional form
of
these distributions, one has only samples from them.  Given samples $(\v x_i,
y_i)$,
it is not straight-forward to estimate $P(z_k)$ or $P(z_k ~|~ k)$, since a
given
value $z_k$ is likely to occur only once, or not at all, since the
sample set is finite.  One possibility is to use histograms of these
distributions.
However, this requires selection of the bin size and the division between
bins, and the result depends on these parameters.  
For this reason, we believe this is an inadequate solution.

The approach suggested by the Kolmogorov-Smirnov test is to compare the cumulative distributions. 
Thus, with $k$ given, one tests the
equality
\begin{equation}
\label{eq:KS-distributions}
\int_0^\sigma P( z_k, k) \, dz_k = \int_0^\sigma z_k \,  P(z_k)\, dz_k ~.
\end{equation}
Writing $\phi_1(\sigma)$ and $\phi_2(\sigma)$ to be the two sides of this
equation, the KS-distance between these two distributions is defined as $\text{KS} = \max_\sigma
|\phi_1(\sigma) - \phi_2(\sigma)|$.
The fact that simply the maximum is used here may suggest a lack of robustness,
but this is a maximum difference between two integrals, so it reflects
an accumulated difference between the two distributions.  

To provide more insights into the KS-distance, let us a consider a case where  
$z_k$ consistently over or under-estimates $P(k ~|~ z_k)$  (which is
usually the case, at least for top-$1$ classification~\citenew{guo2017calibration}), then
$P(k ~|~ z_k) - z_k$ has constant sign for all values of $z_k$.
It follows that $P(z_k, k) - z_k P(z_k)$ has constant sign
and so the maximum value in the KS-distance is achieved when $\sigma = 1$.
In this case, 
\begin{align}
\begin{split}
\text{KS} &= \int_0^1 \big| P(z_k, k) - z_k P(z_k) \big| \, dz_k 
                = \int_0^1 \big|P(k ~|~ z_k ) - z_k  \big| \,P(z_k) \, dz_k~,
\end{split}
\end{align}
which is the expected difference between $z_k$ and $P(k ~|~ z_k)$.
This can be equivalently referred to as the {\em expected calibration error} 
for the class $k$. 

\paragraph{Sampled distributions. }
Given samples $(\v x_i, y_i)_{i=1}^N$, and a fixed $k$, one can estimate these cumulative
distributions by
\begin{equation}
\int_0^\sigma P( z_k, k)\, dz_k \approx \frac{1}{N}  \sum_{i=1}^N \v 1(f_k(\v x_i)
\le \sigma) \times \v 1(y_i = k) ~,
\end{equation}
where $\v 1: {\cal B} \rightarrow \{0, 1\}$ is the function that returns $1$ if
the Boolean
expression is true and otherwise $0$.  Thus, the sum is
simply a count of the number of samples for which $y_i = k$ and $f_k(\v x_i) \le \sigma$,
and so the integral represents the proportion of the data 
satisfying this condition.
Similarly, 
\begin{equation}
 \int_0^\sigma z_k \,  P(z_k)\, dz_k  \approx \frac{1}{N} \,\sum_{i=1}^N \v
1(f_k(\v x_i) \le \sigma) f_k(\v x_i) ~.
\end{equation}
These sums can be computed quickly by sorting the data according to the values
$f_k(\v x_i)$, then defining two sequences as follows.
\begin{align}
\begin{split}
\label{eq:cumsum}
\tilde{h}_0 = h_0 &= 0~, \\
h_i &= h_{i-1} + \v 1(y_i = k) /N~, \\
\tilde{h}_i &= \tilde{h}_{i-1} + f_k(\v x_i) /N~.
\end{split}
\end{align}
The two sequences should be the same, and the metric
\begin{equation}
\text{KS}(f_k) = \max_i \,  |h_i - \tilde{h}_i|~,
\end{equation}
gives a numerical estimate of the similarity, and hence a measure of the degree
of calibration of
$f_k$.  This is essentially a version of the Kolmogorov-Smirnov test for
equality of two distributions.

\paragraph{Remark.} 
All this discussion holds also when $k < 0$, for top-$r$ and within-top-$r$ predictions as discussed in \sect{prelim}.   In
\eq{cumsum},  for instance, $f_{-1}(\v x_i)$ 
means the top score, $f_{-1}(\v x_i) = \max_k (f_k(\v x_i))$, or more generally,
$f_{-r}(\v x_i)$ means the $r$-th top score.  Similarly,
the expression $y_i = -r$ means that $y_i$ is the class that has the
$r$-th top score.
Note when calibrating the top-$1$ score, our method is applied after identifying the top-$1$ score, hence, it does not alter the classification accuracy.

\begin{figure}
\centerline{
\includegraphics[width=\textwidth]{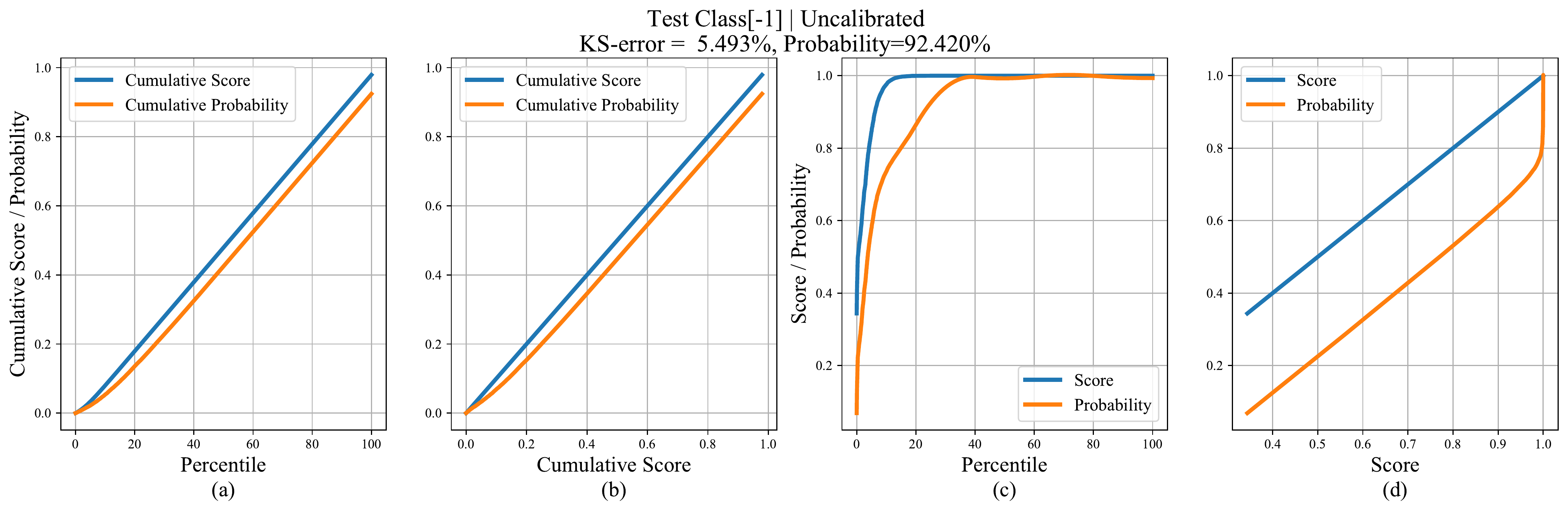}
}
\vspace{-2ex}
\caption{\small \em Calibration graphs for an uncalibrated DenseNet-40
\citenew{huang2017densely} trained on CIFAR-10 for top-1 class with a KS error of $5.5$\%, and top-$1$ accuracy of $92.4$\%  on the
test set. Here (\textbf{a}) shows the plot of cumulative score and probability versus the
fractile of the test set, (\textbf{b}) shows the same information with the
horizontal axis warped so that the cumulative-score graph is a straight line.  
This is created as scatter plots of cumulative (score, score): blue and (score, probability): orange. 
If the network is perfectly calibrated, the probability line will be a straight line
coincident with the (score, score) line. This shows that the network is substantially
overestimating (score) the probability of the computation. (\textbf{c}) and (\textbf{d}) show plots of (non-cumulative) score and probability plotted against fractile, or score.  How these plots are produced is described in \sect{recal-spline}.
}
\label{fig:main-graphs}
\vspace{-2ex}
\end{figure}


\SKIP{
\begin{figure}
\centerline{
\includegraphics[height=3.5in]{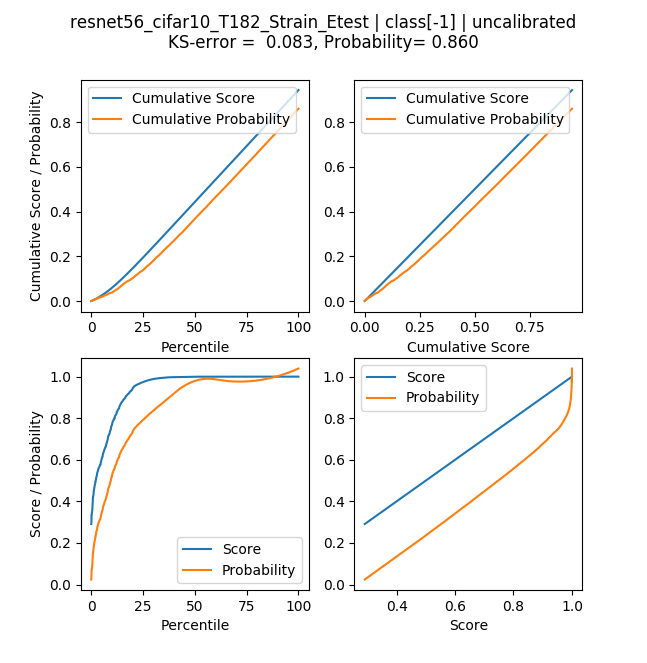} 
}
\caption{\small \aj{better to put them in single row to save space}
\em Calibration graphs for an uncalibrated network, in this case resnet56 
with cifar10 dataset.  The title of the graph shows (reading left to right) that
it is resnet56,
with cifar10 dataset, trained for 182 epochs, using the set ``train'', evaluated
on set ``test'' ($5,000$ images),
reporting results for class $-1$ (the top result), and uncalibrated. 
It has a KS error of $0.083$, and is achieving $0.860$ top-$1$ accuracy on the
test set.
\vspace*{0.05in}
\newline
The top left graph shows the plot of cumulative score and probability versus the
fractile of 
the test set.  The top right graph shows the same information with the
horizontal axis warped
so that the cumulative-score graph is a straight line.  This is created as
scatter plots of (cumulative)
(score, score) and (score, probability).  If the network is perfectly
calibrated, the probability line 
will be a straight line coincident with the (score,score) line.
This shows that the
network is substantially overestimating (score) the probability of the
computation.
\vspace*{0.05in}
\newline
The bottom graphs show plots of (non-cumulative) score and probability plotted
agains
fractile, or score.  How these plots are produced will be described in the
text.
}
\label{fig:main-graphs}
\end{figure}
}

\SKIP{
\begin{figure}
\centerline{
\includegraphics[height=3.5in]{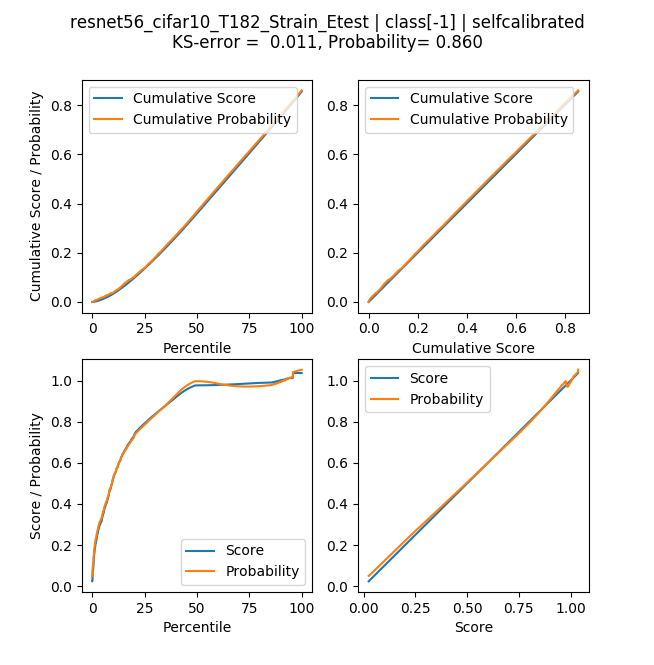} 
}
\caption{
\em This shows the result of the spline calibration method, on the
example given in \fig{main-graphs}.  A recalibration function $\gamma:
\R\rightarrow\R$
is used to adjust the scores, replacing $f_k(x)$ with $\gamma(f_k(x))$.  
As is seen, the network is now almost
perfectly calibrated when tested on the set ``calib'' that was used to
calibrate
it.  
}
\label{fig:resnet-self-calibrated}
\end{figure}
}
\SKIP{ 
\begin{figure}[t]
\centerline{
\includegraphics[width=\textwidth]{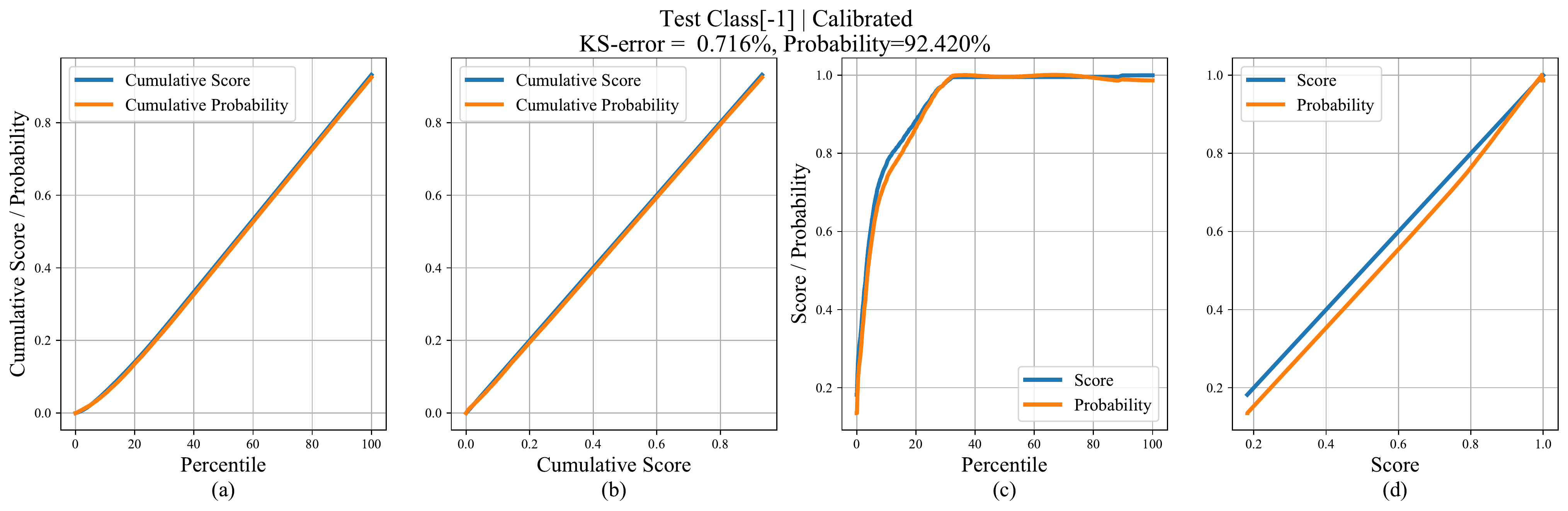} 
}
\caption{
\em Results for the example in \fig{main-graphs} with scores recalibrated
using a recalibration function $\gamma$ computed using set ``calib'' 
and then tested on a further set ``test''.  It is seen that the result is not perfect, but much better than the original results in \fig{main-graphs}. It is also notable that the improvement in calibration is achieved without any loss of accuracy.
}
\label{fig:resnet-calibrated}
\end{figure}
} 

\SKIP{
\begin{figure}
\centerline{
\includegraphics[height=3.65in]{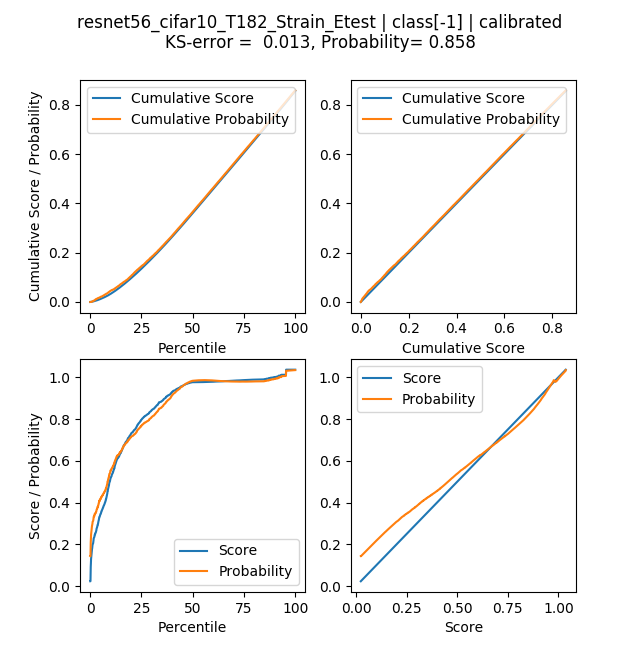} 
}
\caption{
\em Results for the example in \fig{main-graphs} with scores recalibrated
using a recalibration function $\gamma$ computed using set ``calib'' 
and then tested on a further set ``test'', with $5000$ 
samples.  It is seen that the result is not perfect, but much better than the
original results in \fig{main-graphs}.  In the bottom-right graph, the
probability
curve lies appreciably above the score, for values of score up to $0.5$.
However, as seen in the bottom-right graph, this represents less
than $10\%$ of samples.
It is also notable that the improvement in calibration is achieved without
any appreciable loss of accuracy ($0.860$ in the calibrated case, $0.858$
after calibration).  This small difference is explained by variability in
the test set.
}
\label{fig:resnet-calibrated}
\end{figure}
}

\section{Recalibration using Splines}\label{sec:recal-spline}
The function $h_i$ defined in \eq{cumsum} computes an empirical approximation
\begin{equation}
h_i \approx P(Y = k, f_k(X) \le f_k(\v x_i)) ~.
\end{equation}
For convenience, the value of $f_k$ will be referred to as the {\em score}.
We now define a continuous function $h(t)$ for $t \in [0,1]$ by
\begin{equation}\label{eq:ht}
h(t) = P(Y = k, f_k(X) \le s(t)  ) ~,
\end{equation}
where $s(t)$ is the $t$-th fractile score, namely the value that
a proportion $t$ of the scores $f_k(X)$ lie below.
For instance $s(0.5)$ is the median score.
So, $h_i$ is an empirical approximation to $h(t)$ where $t = i/N$. 
We now provide the basic observation that allows us to compute 
probabilities given the scores.

\begin{proposition}\label{pro:spline-deriv}
If $h(t) = P(Y = k, f_k(X) \le s(t))$ as in \eq{ht} where $s(t)$ 
is the $t$-th fractile score, then 
$h'(t) = P(Y = k ~|~ f_k(X) = s(t))$, where $h'(t) = dh/dt$.
\end{proposition}
\begin{proof}
The proof relies on the equality $P(f_k(X) \le s(t)) = t$.
In words,
since $s(t)$ is the 
value that a fraction $t$ of the scores are less than or equal, 
the probability that a score is less than or equal to $s(t)$, 
is (obviously) equal to $t$.
See the supplementary material for a detailed proof.
\end{proof}
Notice $h'(t)$ allows direct conversion from score to probability.
Therefore, our idea is to approximate $h_i$ using a differentiable function and take the derivative which would be our recalibration function.

\SKIP{
\begin{proposition}\aj{don't seem mathematical, may be we'll simply put in text?}
\label{pro:spline-deriv}
If the function $h$ (see \eq{cumsum}) is plotted against the score fractile (top
left 
graph in \fig{main-graphs}, orange line), then the derivative $h'(t)$ (bottom
left graph)
gives the probability $P(k ~|~ f_k(X) = s(t))$.  If probability is plotted
directly
against score, (bottom right graph), then this allows a direct conversion from
score to probability.
\end{proposition}
}
%

\subsection{Spline fitting}
\vspace{-1ex}

The function $h_i$ (shown in \fig{main-graphs}a) is obtained
through sampling
only.  Nevertheless, the sampled graph is smooth and increasing.  There are
various ways to fit a smooth curve to it, so as to take derivatives.  We choose to fit the sampled points $h_i$ to a cubic spline and take its derivative.

Given sample points $(u_i, v_i)_{i=1}^N$ in $\R\times \R$, easily available
references 
show how to fit a smooth spline curve
that passes directly through the points $(u_i, v_i)$. 
A very clear description is given in~\cite{Mckinley_cubicspline}, 
for the case where the points
$u_i$ are equally
spaced.   We wish, however, to fit a spline curve with a small number of knot
points to do a least-squares fit to the points.  
For convenience, this is briefly described here.

A cubic spline $v(u)$ is defined by its values at 
certain knot points $(\hat u_k,  \hat v_k)_{k=1}^K$.
In fact, the value of the curve at any point $u$ can be written as a linear
function
$v(u) = \sum_{k=1}^K a_k(u)\hat v_k = \v a\tr\!(u)\, \hat{\v v}$, 
where the coefficients
$a_k$ depend on $u$.%
\footnote{
Here and elsewhere, notation such as $\v v$ and
$\v a$ denotes the vector of values $v_i$ or $a_k$, as appropriate.
}
Therefore, given a set of further points
$(u_i, v_i)_{i=1}^N$, which may be different from the knot points, and typically
more in number, least-squares spline fitting of the points $(u_i, v_i)$
can be written as a least-squares problem 
$
\min_{\hat{\v v}} \| \m A(\v u)\hat{\v v} - \v v \|^2
$,
which is solved by standard linear least-squares techniques.  Here, the matrix
$\m A$ has dimension $N \times K$ with $N > K$.
Once $\hat{\v v}$ is found, the value of the spline at any further points
$u$ is equal to $v(u) = \v a(u)\tr \hat{\v v}$, a linear combination of the knot-point
values $\hat v_k$.

Since the function is piecewise cubic, with continuous second derivatives,
the first derivative of the spline is computed analytically.
Furthermore, the derivative $v'(u)$ can also be written as a linear  combination
$v'(u) = \v a'(u)\tr \hat{\v v}$, where the coefficients $\v a'(u)$ can be
written explicitly.

Our goal is to fit a spline to a set of data points $(u_i, v_i) = (i/N, h_i)$ 
defined in \eq{cumsum}, in other words, the values $h_i$ plotted
against fractile score.   Then according to \pro{spline-deriv},
the derivative of the spline is equal to $P(k ~|~ f_k(X) = s(t))$.
This allows a direct computation of the conditional probability that
the sample belongs to class $k$.  

Since the derivative of $h_i$ is a probability, one might
constrain the derivative to be in the range $[0,1]$ while fitting splines.
This can be easily incorporated because the derivative of the spline is a 
linear expression in $\hat v_i$.
The spline fitting problem thereby
becomes a linearly-constrained quadratic program (QP).
However, although we tested this, in all the reported experiments, a simple least-squares solver is used without the
constraints.

\SKIP{
\rih{The next two paragraphs could go, since Kartik
says that they do not improve results, and so we quote results
for spline fitting without the linear constraint.  However, keep if there is room. }
In fitting the spline, 
there are two requirements: 1) 
the derivative should lie in the range $[0, 1]$ (since it is a probability), 
and 2) it should be monotonically increasing (probability of 
belonging to class $k$ should
increase with the score).  
This latter condition is an empirical requirement,
and could be omitted.

It is enough to enforce these conditions at the knot points, or alternatively
at every point $u_i$.  Since the derivative (and also the second derivative)
of the spline is a linear expression in the $\hat v_i$, this leads to linear
constraints on the optimization problem.  Therefore, the spline fitting
problem becomes a linearly-constrained quadratic program. 
We use the qp solver from \verb|cvxopt|~\cite{andersen2013cvxopt} for optimization.
}

\begin{figure}
\centerline{
\includegraphics[width=\textwidth]{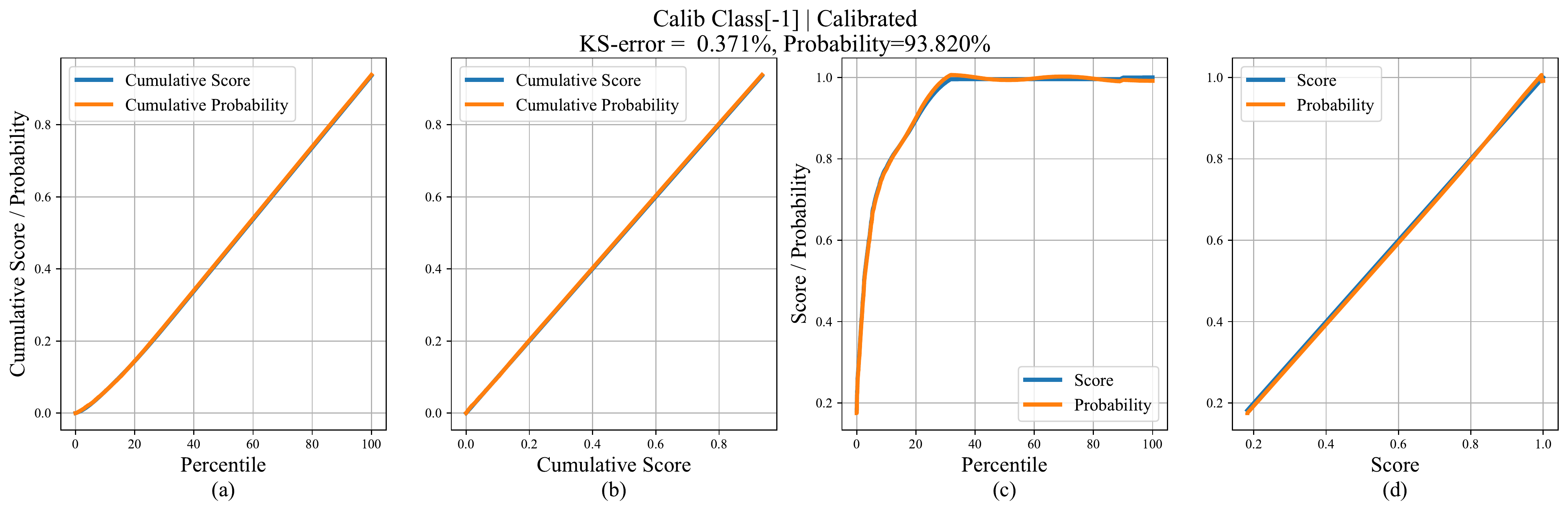} 
}
\vspace{-0.8ex}
\centerline{
\includegraphics[width=\textwidth]{figures/cifa10-densenet40/spline_calibrated_test_class-1_KS.pdf} 
}
\vspace{-2ex}
\caption{
\em The result of the spline calibration method, on the
example given in \fig{main-graphs} for top-$1$ calibration. 
A recalibration function $\gamma:
\R\rightarrow\R$
is used to adjust the scores, replacing $f_k(\v x)$ with $\gamma(f_k(\v x))$ 
(see \sect{recal}).  
As is seen, the network is now almost
perfectly calibrated when tested on the ``calibration'' set (\textbf{top row}) 
used to calibrate it. 
In \textbf{bottom row}, the recalibration function is tested
on a further set ``test''.  It is seen that the result is not perfect, but much better than the one in \fig{main-graphs}d. It is also notable that the improvement in calibration is achieved without any loss of accuracy.
}
\label{fig:densenet-calibrated}
\vspace{-1ex}
\end{figure}

\subsection{Recalibration}\label{sec:recal}
\vspace{-1ex}
We suppose that the classifier $f = f_\theta$ is fixed, through training on 
the training set.
Typically, if the classifier is tested on
the training set, it is very close to being calibrated.  However, if a
classifier
$f$ is then tested on a different set of data, it may be substantially
mis-calibrated.
See \fig{main-graphs}.  

Our method of calibration is to find a further mapping $\gamma: [0,1]
\rightarrow [0,1]$, 
such that $\gamma \circ f_k$ is calibrated. 
This is easily obtained from the 
 direct mapping from score $f_k(\v x)$ to
$P(k ~|~ f_k(\v x))$ (refer to \fig{main-graphs}d).
In equations, 
$
\gamma(\sigma) = h'(s^{-1}(\sigma))
$.
The function $h'$ is known analytically, from fitting a spline to $h(t)$ and
taking its derivative.
The function $s^{-1}$ is a mapping from the given score $\sigma$
to its fractile $s^{-1}(\sigma)$.
Note that, a held out calibration set is used to fit the splines and 
the obtained recalibration function $\gamma$ is evaluated on an unseen 
test set.

To this end, given a sample $\v x$ from the test set with 
$f_k(\v x) = \sigma$, 
one can compute $h'(s^{-1}(\sigma))$ directly in one
step by interpolating
its value between the values of $h'(f_k (\v x_i))$ and $h'(f_k (\v x_{i+1}))$ where
$\v x_i$ and $\v x_{i+1}$
are two samples from the calibration set, with closest scores 
on either side of $\sigma$. 
Assuming the samples in the calibration set are ordered, the 
samples $\v x_i$ and $\v x_{i+1}$ can be quickly located using binary search.
Given a reasonable number of samples in the calibration set, 
(usually in the order of thousands), this
can be very accurate.
In our experiments, improvement in calibration is observed in the test set
with no difference to the accuracy of the network (refer to \fig{densenet-calibrated}d).
In practice, spline fitting is much faster than one forward pass through the network and it is highly scalable compared to learning based calibration methods.



\section{Related Work}
\SKIP{
Obtaining well calibrated predictors has been studied by meteorologists and
statisticians for several decades. 
Brier~\cite{brier1950verification} studied the verification of weather forecasts
when they are expressed in terms of probabilities and introduced the famous
Brier score. 
Many follow up works studied the concept of
reliability~\cite{murphy1977reliability}, proper scoring
rules~\cite{winkler1968good}, and calibration and
refinement~\cite{degroot1983comparison} in a similar context.
}

\paragraph{Modern calibration methods.}
In recent years, neural networks are shown to overfit to the Negative
Log-Likelihood (NLL) loss and in turn produce overconfident predictions which is cited as 
the main reason for miscalibration~\citenew{guo2017calibration}.
To this end, modern calibration methods can be broadly categorized into 1) methods
that adapt the training procedure of the classifier, and 2) methods that learn a
recalibration function post training.
Among the former, the main idea is to increase the entropy of the classifier to avoid
overconfident predictions, which is accomplished via modifying the training
loss~\citenew{kumar2018trainable,mukhoti2020calibrating,seo2019learning}, label
smoothing~\citenew{muller2019does,pereyra2017regularizing}, and data augmentation
techniques~\citenew{thulasidasan2019mixup,yun2019cutmix,mixup}.

On the other hand, we are interested in calibrating an already trained classifier that eliminates the need for training from scratch. 
In this regard, a popular approach is Platt scaling~\citenew{platt1999probabilistic} which
transforms the outputs of a binary
classifier into probabilities by fitting a
scaled logistic function on a held out calibration set.
Similar approaches on binary classifiers include Isotonic
Regression~\citenew{zadrozny2001obtaining}, histogram and
Bayesian binning~\citenew{naeini2015obtaining,zadrozny2001obtaining}, and Beta
calibration~\citenew{kull2017beta}, which are later extended to the multiclass
setting~\citenew{guo2017calibration,kull2019beyond,zadrozny2002transforming}.
Among these, the most popular method is temperature
scaling~\citenew{guo2017calibration},
which learns a single scalar on a held out set to calibrate the
network predictions.
Despite being simple and one of the early works, temperature scaling is the method to
beat in calibrating modern networks.
Our approach falls into this category, however, as opposed to minimizing a loss
function, we obtain a recalibration function via spline-fitting, which directly maps the
classifier outputs to the calibrated probabilities.

\paragraph{Calibration measures.} 
Expected Calibration Error (ECE)~\citenew{naeini2015obtaining} is the most popular
measure in
the literature, however, it has a weakness that the resulting calibration error
depends on the
histogram binning scheme such as the bin endpoints and the number of bins.
Even though, some improvements have been
proposed~\citenew{nixon2019measuring,vaicenavicius2019evaluating}, the binning
scheme has not been eliminated and it is recently shown that
any binning scheme leads to underestimated calibration
errors~\citenew{kumar2019verified,widmann2019calibration}.
Note that, there are binning-free metrics exist such as Brier
score~\citenew{brier1950verification}, NLL, and kernel based metrics for the
multiclass setting~\citenew{kumar2018trainable,widmann2019calibration}.
Nevertheless, the Brier score and NLL measure a combination of calibration error and
classification error (not just the calibration which is the focus).
Whereas kernel based metrics, besides being computationally expensive, 
measure the calibration of the predicted probability vector rather than the 
classwise calibration error~\citenew{kull2019beyond} 
(or top-$r$ prediction) which is typically 
the quantity of interest. 
To this end, we introduce a binning-free calibration measure based on the
classical KS-test, which has the same benefits as ECE and provides effective
visualizations similar to reliability diagrams.
Furthermore, KS error can be shown to be a special case of kernel based
measures~\citenew{gretton2012kernel}.

\section{Experiments}

\SKIP{
\begin{table}[t!]
\scriptsize
\begin{tabular}{ll|cccccc}
\toprule
Dataset &
  Model &
  Uncalibrated &
  Temp. Scaling &
  Vector Scaling &
  MS-ODIR &
  Dir-ODIR &
  Ours (Spline) \\ \midrule
\multirow{5}{*}{CIFAR-10} &
  Resnet-110 &
  0.04750 &
  \underline{0.00916} &
  0.00996 &
  0.00977 &
  0.01060 &
  \textbf{0.00643} \\
 &
  Resnet-110-SD &
  0.04102 &
  0.00362 &
  0.00430 &
  \underline{0.00358} &
  0.00389 &
  \textbf{0.00269} \\
  &
  DenseNet-40 &
  0.05493 &
  0.00900 &
  \underline{0.00890} &
  0.00897 &
  0.01057 &
  \textbf{0.00773} \\
 &
  Wide Resnet-32 &
  0.04475 &
  0.00296 &
  \textbf{0.00267} &
  0.00305 &
  \underline{0.00291} &
  0.00367 \\
 &
  Lenet-5 &
  0.05038 &
  0.00799 &
  0.00839 &
  \underline{0.00646} &
  0.00854 &
  \textbf{0.00348} \\ \midrule
\multirow{5}{*}{CIFAR-100} &
  Resnet-110 &
  0.18481 &
  \underline{0.01489} &
  0.01827 &
  0.02845 &
  0.02575 &
  \textbf{0.00575} \\
 &
  Resnet-110-SD &
  0.15832 &
  \textbf{0.00748} &
  0.01303 &
  0.03572 &
  0.01645 &
  \underline{0.01028} \\
 &
  DenseNet-40 &
  0.21156 &
  \textbf{0.00304} &
  0.00483 &
  0.02350 &
  0.00618 &
  \underline{0.00454} \\
 &
  Wide Resnet-32 &
  0.18784 &
  \underline{0.01130} &
  0.01642 &
  0.02524 &
  0.01788 &
  \textbf{0.00930} \\
 &
  Lenet-5 &
  0.12117 &
  0.01215 &
  \underline{0.00768} &
  0.01047 &
  0.02125 &
  \textbf{0.00391} \\
\midrule
\multirow{2}{*}{ImageNet} &
  Densenet-161 &
  0.05721 &
  \underline{0.00744} &
  0.02014 &
  0.04723 &
  0.03103 &
  \textbf{0.00406} \\
 &
  Resnet-152 &
  0.06544 &
  \underline{0.00791} &
  0.01985 &
  0.05805 &
  0.03528 &
  \textbf{0.00441} \\
\midrule
SVHN &
  Resnet-152-SD &
  0.00852 &
  \textbf{0.00552} &
  0.00570 &
  0.00573 &
  0.00607 &
  \underline{0.00556} \\
\bottomrule
\end{tabular}
    \vspace{1ex}

    \caption{\em KS error for top-1 prediction (lowest in bold and second lowest underlined) on various image classification datasets and models with different calibration methods. Note, our method consistently reduces calibration error to $<\textbf{1}$\% in all experiments, outperforming state-of-the-art methods.
    }
    \label{tab:res_PKSE}  
    \vspace{-2ex}
\end{table}
}

\begin{table}[t!]
\scriptsize
\begin{tabular}{ll|cccccc}
\toprule
Dataset &
  Model &
  Uncalibrated &
  Temp. Scaling &
  Vector Scaling &
  MS-ODIR &
  Dir-ODIR &
 \textbf{Ours (Spline)} \\ \midrule
\multirow{5}{*}{CIFAR-10} &
  Resnet-110 &
  4.750 &
  \underline{0.916} &
  0.996 &
  0.977 &
  1.060 &
  \textbf{0.643} \\
 &
  Resnet-110-SD &
  4.102 &
  0.362 &
  0.430 &
  \underline{0.358} &
  0.389 &
  \textbf{0.269} \\
  &
  DenseNet-40 &
  5.493 &
  0.900 &
  \underline{0.890} &
  0.897 &
  1.057 &
  \textbf{0.773} \\
 &
  Wide Resnet-32 &
  4.475 &
  0.296 &
  \textbf{0.267} &
  0.305 &
  \underline{0.291} &
  0.367 \\
 &
  Lenet-5 &
  5.038 &
  0.799 &
  0.839 &
  \underline{0.646} &
  0.854 &
  \textbf{0.348} \\ \midrule
\multirow{5}{*}{CIFAR-100} &
  Resnet-110 &
  18.481 &
  \underline{1.489} &
  1.827 &
  2.845 &
  2.575 &
  \textbf{0.575} \\
 &
  Resnet-110-SD &
  15.832 &
  \textbf{0.748} &
  1.303 &
  3.572 &
  1.645 &
  \underline{1.028} \\
 &
  DenseNet-40 &
  21.156 &
  \textbf{0.304} &
  0.483 &
  2.350 &
  0.618 &
  \underline{0.454} \\
 &
  Wide Resnet-32 &
  18.784 &
  \underline{1.130} &
  1.642 &
  2.524 &
  1.788 &
  \textbf{0.930} \\
 &
  Lenet-5 &
  12.117 &
  1.215 &
  \underline{0.768} &
  1.047 &
  2.125 &
  \textbf{0.391} \\
\midrule
\multirow{2}{*}{ImageNet} &
  Densenet-161 &
  5.721 &
  \underline{0.744} &
  2.014 &
  4.723 &
  3.103 &
  \textbf{0.406} \\
 &
  Resnet-152 &
  6.544 &
  \underline{0.791} &
  1.985 &
  5.805 &
  3.528 &
  \textbf{0.441} \\
\midrule
SVHN &
  Resnet-152-SD &
  0.852 &
  \textbf{0.552} &
  0.570 &
  0.573 &
  0.607 &
  \underline{0.556} \\
\bottomrule
\end{tabular}
    \vspace{1ex}

    \caption{\em KS Error (in \%) for top-1 prediction (with lowest in bold and second lowest underlined) on various image classification datasets and models with different calibration methods. Note, our method consistently reduces calibration error to $<\textbf{1}$\% in almost all experiments, outperforming state-of-the-art methods.
    }
    \label{tab:res_PKSE}  
    \vspace{-2ex}
\end{table}
\SKIP{
\begin{table}[t!]
\scriptsize
\begin{tabular}{ll|cccccc}
\toprule
Dataset &
  Model &
  Uncalibrated &
  Temp. Scaling &
  Vector Scaling &
  MS-ODIR &
  Dir-ODIR &
  Ours (Spline) \\ \midrule
\multirow{5}{*}{CIFAR-10} & Resnet-110 & 0.03011 & 0.00947 & 0.00948 & \underline{0.00598} & 0.00953 & \textbf{0.00347} \\
 & Resnet-110-SD & 0.02716 & 0.00478 & 0.00486 & \underline{0.00401} & 0.00500 & \textbf{0.00310} \\
 & DenseNet-40 & 0.03342 & \textbf{0.00535} & \underline{0.00543} & 0.00598 & 0.00696 & 0.00695 \\
 & Wide Resnet-32 & 0.02669 & 0.00426 & \underline{0.00369} & 0.00412 & 0.00382 & \textbf{0.00364} \\
 & Lenet-5 & 0.01708 & \underline{0.00367} & \textbf{0.00279} & 0.00409 & 0.00426 & 0.00837 \\ \midrule
\multirow{5}{*}{CIFAR-100} 
& Resnet-110 & 0.04731 & 0.01401 & 0.01436 & \underline{0.00961} & 0.01269 & \textbf{0.00371} \\
 & Resnet-110-SD & 0.03923 & \textbf{0.00315} & \underline{0.00481} & 0.00772 & 0.00506 & 0.00595 \\
 & DenseNet-40 & 0.05803 & 0.00305 & 0.00653 & \underline{0.00219} & \textbf{0.00135} & 0.00903 \\
 & Wide Resnet-32 & 0.05349 & 0.00790 & 0.01095 & \underline{0.00646} & 0.00845 & \textbf{0.00372} \\
 & Lenet-5 & 0.02615 & 0.00571 & \underline{0.00439} & \textbf{0.00324} & 0.00799 & 0.00587 \\
\midrule
\multirow{2}{*}{ImageNet} 
& Densenet-161 & 0.01689 & \underline{0.01044} & 0.01166 & 0.01288 & 0.01321 & \textbf{0.00178} \\
 & Resnet-152 & 0.01793 & \underline{0.01151} & 0.01264 & 0.01660 & 0.01430 & \textbf{0.00580} \\
\midrule
SVHN 
& Resnet-152-SD 
& 0.00373 & 0.00226 & \textbf{0.00216} & 0.00973 & \underline{0.00218} & 0.00492 \\
\bottomrule
\end{tabular}
    \vspace{1ex}

    \caption{\em KS Error for top-2 prediction\aj{updated with within top-2 results} (with lowest in bold and second lowest underlined) on various image classification datasets and models with different calibration methods. Again, our method consistently reduces calibration error to $<\textbf{1}$\% in all experiments.}
    \label{tab:res_KSE2}  
\end{table}
}

\begin{table}[t!]
\scriptsize
\begin{tabular}{ll|cccccc}
\toprule
Dataset &
  Model &
  Uncalibrated &
  Temp. Scaling &
  Vector Scaling &
  MS-ODIR &
  Dir-ODIR &
  \textbf{Ours (Spline)} \\ \midrule
\multirow{5}{*}{CIFAR-10} & Resnet-110 & 3.011 & 0.947 & 0.948 & \underline{0.598} & 0.953 & \textbf{0.347} \\
 & Resnet-110-SD & 2.716 & 0.478 & 0.486 & \underline{0.401} & 0.500 & \textbf{0.310} \\
 & DenseNet-40 & 3.342 & \textbf{0.535} & \underline{0.543} & 0.598 & 0.696 & 0.695 \\
 & Wide Resnet-32 & 2.669 & 0.426 & \underline{0.369} & 0.412 & 0.382 & \textbf{0.364} \\
 & Lenet-5 & 1.708 & \underline{0.367} & \textbf{0.279} & 0.409 & 0.426 & 0.837 \\ \midrule
\multirow{5}{*}{CIFAR-100} 
& Resnet-110 & 4.731 & 1.401 & 1.436 & \underline{0.961} & 1.269 & \textbf{0.371} \\
 & Resnet-110-SD & 3.923 & \textbf{0.315} & \underline{0.481} & 0.772 & 0.506 & 0.595 \\
 & DenseNet-40 & 5.803 & 0.305 & 0.653 & \underline{0.219} & \textbf{0.135} & 0.903 \\
 & Wide Resnet-32 & 5.349 & 0.790 & 1.095 & \underline{0.646} & 0.845 & \textbf{0.372} \\
 & Lenet-5 & 2.615 & 0.571 & \underline{0.439} & \textbf{0.324} & 0.799 & 0.587 \\
\midrule
\multirow{2}{*}{ImageNet} 
& Densenet-161 & 1.689 & \underline{1.044} & 1.166 & 1.288 & 1.321 & \textbf{0.178} \\
 & Resnet-152 & 1.793 & \underline{1.151} & 1.264 & 1.660 & 1.430 & \textbf{0.580} \\
\midrule
SVHN 
& Resnet-152-SD 
& 0.373 & 0.226 & \textbf{0.216} & 0.973 & \underline{0.218} & 0.492 \\
\bottomrule
\end{tabular}
    \vspace{1ex}

    \caption{\em KS Error (in \%) for top-2 prediction (with lowest in bold and second lowest underlined) on various image classification datasets and models with different calibration methods. Again, our method consistently reduces calibration error to $<1$\% (less then $0.7$\%, except for one case), in all experiments, the only one of the methods to achieve this.}
    \label{tab:res_KSE2}  
\end{table}

\paragraph{Experimental setup.}
We evaluate our proposed calibration method on four different image-classification datasets namely CIFAR-10/100~\citenew{krizhevsky2009learning}, SVHN~\citenew{netzer2011reading} and ImageNet~\citenew{deng2009imagenet} using LeNet~\citenew{lecun1998gradient}, ResNet~\citenew{he2016deep}, ResNet with stochastic depth~\citenew{huang2017densely}, Wide ResNet~\citenew{zagoruyko2016wide} and DenseNet~\citenew{huang2017densely} network architectures against state-of-the-art
methods that calibrate post training.
We use the pretrained network logits\footnote{Pre-trained network logits are obtained from \url{https://github.com/markus93/NN\_calibration}.} for spline fitting where we choose validation set as the calibration set, similar to the standard practice. Our final results for calibration are then reported on the test set of all datasets. Since ImageNet does not comprise the validation set, test set is divided into two halves: calibration set and test set.
We use the natural cubic spline fitting method (that is, cubic splines 
with linear run-out) with $6$ knots for all our experiments. 
Further experimental details are provided in the supplementary. For baseline methods namely: Temperature scaling, Vector scaling, Matrix scaling with ODIR (Off-diagonal and Intercept Regularisation), and Dirichlet calibration, we use the implementation of Kull~\etal~\citenew{kull2019beyond}. 

\paragraph{Results.} We provide comparisons of our method using proposed KS error for the top most prediction against state-of-the-art calibration methods namely temperature scaling~\citenew{guo2017calibration}, vector scaling, MS-ODIR, and Dirichlet Calibration (Dir-ODIR)~\citenew{kull2019beyond} in \tabref{tab:res_PKSE}. Our method reduces calibration error to $1\%$ in almost all experiments performed on different datasets without any loss in accuracy. It clearly reflects the efficacy of our method irrespective of the scale of the dataset as well as the depth of the network architecture. It consistently performs better than the recently introduced Dirichlet calibration and Matrix scaling with ODIR~\citenew{kull2019beyond} in all the experiments. 
Note this is consistent with the top-1 calibration results reported in Table~15 of~\citenew{kull2019beyond}.
The closest competitor to our method is temperature scaling, against which our method performs better in 9 out of 13 experiments. Note, in the cases where temperature scaling outperforms our method, the gap in KS error between the two methods is marginal ($<\textbf{0.3}\%$) and our method is the second best. We
provide comparisons using other calibration metrics in the supplementary.

From the practical point of view, it is also important for a network to be calibrated for top second/third predictions and so on. We thus show comparisons for top-2 prediction KS error in \tabref{tab:res_KSE2}. An observation similar to the one noted in \tabref{tab:res_PKSE} can be made for the top-2 predictions as well. Our method achieves $<\textbf{1}$\% calibration error in all the experiments. It consistently performs well especially for experiments performed on large scale ImageNet dataset where it sets new \emph{state-of-the-art for calibration}. We would like to emphasize here, though for some cases Kull~\etal~\citenew{kull2019beyond} and Vector Scaling perform better than our method in terms of top-2 KS calibration error, overall (considering both top-1 and top-2 predictions) our method performs better.


\section{Conclusion}
In this work, we have introduced a binning-free calibration metric based on 
the Kolmogorov-Smirnov test to measure classwise or (within)-top-$r$ calibration
errors. 
Our KS error eliminates the shortcomings of the popular ECE measure and its
variants while accurately measuring the expected calibration error and provides
effective visualizations similar to reliability diagrams.
Furthermore, we introduced a simple and effective calibration method based on
spline-fitting which does not involve any learning and yet consistently yields 
the lowest calibration error in the majority of our experiments.
We believe, the KS metric would be of wide-spread use to measure classwise calibration and our spline method would inspire learning-free approaches to neural network calibration.
We intend to focus on calibration beyond classification problems as future work.





\fi

\iffinal 
\section{Acknowledgements}
The work is supported by the Australian Research Council Centre of Excellence for Robotic Vision (project number CE140100016). We would also like to thank Google Research and Data61, CSIRO for their support.
\SKIP{
\section*{Broader Impact}

Predicting probability estimates that reflect the true likelihood of being correct, also known as confidence calibration, is essential when neural networks--or any other predictive system for that matter--are employed in safety-critical applications. In such contexts, the downstream decision making critically depends on the predicted probabilities. Consequently, neural networks should also provide a calibrated confidence measure and not just a prediction.  Safety-critical domains where the methodology is broadly applicable include systems for medical diagnosis, self-driving, authentication, fraud detection or stock-market analysis, among many others. 
This work proposes a novel methodology to improve it towards potentially broad practical impact. The neural network calibration process can also be seen as reducing certain learning biases that eventually lead to incorrect probabilities and decision-making failure. In large-scale decision making systems, classification networks must not just be accurate, but should also have built in mechanisms that would indicate when they are likely to fail. Enhanced confidence models would help establish trust with the user, support the integration into larger probabilistic systems, and enable superior human interpretability in the long-run.
We do not foresee any direct negative impacts, however, it may contribute to the societal impacts of deep learning technology, both positive and negative.
}
\fi
\ifarxiv
\clearpage
\appendices

\fi

\bibliographystyle{iclr2021_conference}
\bibliography{spline-mainpaper}
\appendices

\end{document}